\documentclass{article}

\usepackage{microtype}
\usepackage{graphicx}
\usepackage{subcaption}
\usepackage{booktabs} 
\usepackage{multirow}
\usepackage{xcolor}
\usepackage{enumitem}
\usepackage{hyperref}

\usepackage[preprint]{icml2026}

\usepackage{amsmath}
\usepackage{amssymb}
\usepackage{mathtools}
\usepackage{amsthm}

\usepackage{array}
\usepackage[capitalize,noabbrev]{cleveref}

\theoremstyle{plain}
\newtheorem{theorem}{Theorem}[section]

\theoremstyle{definition}

\theoremstyle{remark}

\usepackage[textsize=tiny]{todonotes}
\usepackage[table]{xcolor} 
\definecolor{myteal}{HTML}{287878}

\icmltitlerunning{ScDiVa: Masked Discrete Diffusion for Joint Modeling of Single-Cell Identity and Expression}

\begin{document}

\twocolumn[
  \icmltitle{ScDiVa: Masked Discrete Diffusion for \\ 
  Joint Modeling of Single-Cell Identity and Expression}



  \icmlsetsymbol{equal}{*}

\begin{icmlauthorlist}
    \icmlauthor{Mingxuan Wang}{gsai,stat}
    \icmlauthor{Cheng Chen}{cre}
    \icmlauthor{Gaoyang Jiang}{hust}
    \icmlauthor{Zijia Ren}{math}
    \icmlauthor{Chuangxin Zhao}{baai}
    \icmlauthor{Lu Shi\textsuperscript{*}}{cre}
    \icmlauthor{Yanbiao Ma\textsuperscript{*}}{gsai}  
  \end{icmlauthorlist}
    
    \icmlaffiliation{gsai}{Gaoling School of Artificial Intelligence, Renmin University of China, Beijing, China}
    
    \icmlaffiliation{stat}{School of Statistics, Renmin University of China, Beijing, China}
    
    \icmlaffiliation{cre}{CRE Life Institute, China}
    
    \icmlaffiliation{hust}{School of Computer Science and Technology, Huazhong University of Science and Technology, Wuhan, China}

    \icmlaffiliation{math}{School of Mathematics, Jilin University, Changchun, China}
    
    \icmlaffiliation{baai}{Beijing Academy of Artificial Intelligence,Beijing,China}
    
  \icmlcorrespondingauthor{Yanbiao Ma}{ybma1998@ruc.edu.cn}

  \icmlkeywords{Machine Learning, ICML}

  \vskip 0.3in
]



\printAffiliationsAndNotice{}  

\begin{abstract}
Single-cell RNA-seq profiles are high-dimensional, sparse, and unordered, causing autoregressive generation to impose an artificial ordering bias and suffer from error accumulation. To address this, we propose scDiVa, a masked discrete diffusion foundation model that aligns generation with the dropout-like corruption process by defining a continuous-time forward masking mechanism in token space. ScDiVa features a bidirectional denoiser that jointly models discrete gene identities and continuous values, utilizing entropy-normalized serialization and a latent anchor token to maximize information efficiency and preserve global cell identity. The model is trained via depth-invariant time sampling and a dual denoising objective to simulate varying sparsity levels while ensuring precise recovery of both identity and magnitude. Pre-trained on 59 million cells, scDiVa achieves strong transfer performance across major benchmarks, including batch integration, cell type annotation, and perturbation response prediction. These results suggest that masked discrete diffusion serves as a biologically coherent and effective alternative to autoregression.
\end{abstract}

\section{Introduction}

The rapid advancement of single-cell RNA sequencing (scRNA-seq)~\cite{regev2017human} has necessitated scalable ``Foundation Models'' capable of extracting universal cellular representations from massive global datasets. Inspired by Large Language Models (LLMs), Transformer-based methods such as scBERT \cite{yang2022scbert}, Geneformer \cite{theodoris2023transfer}, and scGPT \cite{cui2024scgpt} treat genes as discrete ``tokens'', demonstrating significant potential in cross-task transfer through self-supervised learning.

However, directly applying NLP paradigms to genomics faces a fundamental Structural Mismatch. First, gene expression profiles are high-dimensional, unordered multisets; Autoregressive (AR) models enforce an artificial sequential order that disrupts symmetric gene regulatory interactions \cite{cui2024scgpt}. Second, existing tokenization strategies struggle to balance sequence length with sparsity, often capturing binary gene presence but failing to reconstruct precise continuous expression intensities (numerical fidelity). Meanwhile, continuous generative models like scVI \cite{lopez2018deep} (VAEs) tend to produce over-smoothed results, and existing diffusion models \cite{luo2024scdiffusion} relying on Gaussian noise fail to model the intrinsic discrete event structure and stochastic dropout of single-cell data.

To address these ``sequence-set'' and ``discrete-continuous'' challenges, we propose Single-cell Masked Diffusion for Identity \& Value (\textbf{scDiVa}), a generative foundation model based on a Masked Discrete Diffusion framework \cite{austin2021structured}. Unlike traditional Gaussian-based methods, we establish a mathematical isomorphism between the forward diffusion process and sequencing ``technical dropout''. By utilizing an absorbing \texttt{[MASK]} state to simulate random signal loss, scDiVa jointly learns gene regulatory topology and quantitative expression manifolds within a unified and robust probabilistic framework.

ScDiVa incorporates Entropy-Normalized Serialization to maximize information density effectively handling long-tail sparsity. Furthermore, we design a Dual Denoising Loss that simultaneously constrains topological classification and dosage regression. This architecture leverages bidirectional non-causal context, achieving precise recovery of cellular states in both ``Rank'' and ``Value'' dimensions while eliminating artificial sequential bias.

Our contributions are threefold:
\vspace{-1mm}
\begin{itemize}[left=-0.5pt]
    \item We propose scDiVa, a foundation model grounded in masked discrete diffusion. By establishing an isomorphism between forward diffusion and biological dropout, it resolves inherent continuous modeling biases and ordering artifacts. Optimized with SwiGLU and RoPE, the framework aligns with physical data properties to robustly parse gene regulatory networks.
    
    \item We employ Entropy-Normalized Serialization to prioritize discriminative features amidst high-dimensional sparsity. Furthermore, a Dual Denoising Loss simultaneously optimizes gene identity and dosage, enabling high-fidelity recovery in Rank and Value dimensions.

    \item We introduce a Depth-Invariant Sampling strategy that models training as inverse physical sequencing. By treating diffusion steps as reciprocal sequencing depth, scDiVa projects cells onto a unified Canonical Latent Manifold, achieving intrinsic robustness to depth variations without \textbf{requiring any} external batch correction.
\end{itemize}

\section{Related Work}

\subsection{LLM-Based Single-Cell Pre-training }
Inspired by pre-training in Natural Language Processing (NLP), single-cell foundation models typically represent each cell as a sequence or a finite set of gene tokens. These models employ self-supervised objectives, such as masked prediction, to learn contextual dependencies and universal embeddings. For instance, scBERT is pre-trained on large-scale unlabeled scRNA-seq data, demonstrating advantages in few-shot fine-tuning and offering interpretability cues via attention weights~\cite{yang2022scbert}. Similarly, Geneformer emphasizes learning contextual dependencies across massive corpora and leverages transferability for diverse network biology predictions~\cite{theodoris2023machine,theodoris2023transfer}. Building on this, scGPT extends generative pre-training to large-scale single-cell and even multi-omics data, systematically validating its transfer efficacy in tasks such as integration and annotation~\cite{cui2024scgpt}. Meanwhile, scFoundation enhances universal cell and gene representations by utilizing a larger corpus and a more comprehensive gene space, followed by evaluation across multiple tasks~\cite{hao2024large}. A critical tension inherent in this approach lies in the fact that tokenization necessitates a complex trade-off among finite sequence length, extreme sparsity, and long-tail dynamic ranges. Consequently, these methods often excel at capturing the identity structure concerning which gene events occur, yet they require additional mechanisms to compensate for the numerical fidelity of continuous expression intensities.

\subsection{Continuous Probabilistic Representation Learning}
Another primary stream focuses on probabilistic generative modeling centered on continuous count distributions. A representative example is scVI, which unifies sequencing noise, batch effects, and latent space inference via Variational Autoencoders (VAEs), serving as a foundational framework for denoising and integration~\cite{lopez2018deep}. In multi-omics scenarios, totalVI further incorporates modalities such as RNA and proteins into a unified probabilistic model, demonstrating the transferability of combining continuous distribution modeling with a unified latent space~\cite{gayoso2021joint}. Concurrently, certain large-scale pre-training methods learn representations directly in the continuous domain for multi-task transfer. For instance, CellFM reports competitive results across various tasks, including perturbation-related ones, while emphasizing a scalable and efficient backbone and training scheme~\cite{zeng2025cellfm}. However, a common risk associated with continuous approaches is that if the objective function primarily encourages numerical reconstruction or regression, the model tends to heavily bias towards smooth or averaged interpretations. This results in a lack of explicit generative pressure regarding which specific gene events should be recovered when missing, thereby limiting capabilities in controllable completion and structured generation.

\subsection{Diffusion Models and Masked Generation}
Diffusion models offer a unified perspective for both continuous and discrete generation and have begun to be applied to single-cell generation and imputation. For example, scDiffusion utilizes conditional diffusion to generate high-fidelity expression profiles~\cite{luo2024scdiffusion}, while scVAEDer explores combining diffusion with VAEs to enhance generation and reconstruction stability~\cite{sadria2025scvaeder}. However, most single-cell diffusion models are still based on continuous noise assumptions, often requiring additional designs to handle the discrete structure of event occurrence or disappearance. In contrast, discrete diffusion (D3PM) and masked generation (MaskGIT) define masking and reverse denoising directly within the discrete state space. These methods interpret the recovery of masked positions as a single-step prediction in the reverse process~\cite{austin2021structured,chang2022maskgit}. Recent efforts to simplify and unify discrete masked diffusion objectives have further enhanced training stability and conceptual consistency~\cite{shi2024simplified}. In the language domain, LLaDA has further validated that diffusion-based pre-training, characterized by a random masking forward process combined with Transformer-based reverse prediction, can serve as an alternative paradigm to autoregression. This approach emphasizes closed-loop consistency from the training objective to generative inference~\cite{nie2025large}.

\section{ScDiVa: Diffusion Foundation Model for Single Cells}
\label{sec:method}


ScDiVa is a generative foundation model designed to address the structural mismatch between autoregressive sequence modeling \cite{floridi2020gpt} and the unordered, sparse nature of single-cell transcriptomic data \cite{cui2024scgpt}. Rather than imposing an artificial gene ordering, scDiVa models generation as a bidirectional denoising process over masked gene tokens. By formulating generation through a masked discrete diffusion process with an absorbing state \cite{austin2021structured,gu2022vector}, the model naturally aligns with stochastic signal dropout observed in single-cell sequencing \cite{lopez2018deep}. This framework enables joint modeling of gene identity and expression magnitude within a unified probabilistic formulation. A detailed biological interpretation of this formulation is provided in Appendix \ref{app:bio_diffusion}.

\subsection{Problem Formulation and Diffusion Modeling}
\label{subsec:diffusion_principle}

Traditional continuous diffusion models typically operate in Euclidean space by adding Gaussian noise \cite{ho2020denoising}. This introduces an inductive bias of ordinality, where distances between values are assumed to be semantically meaningful. Such an assumption is incompatible with the categorical nature of discrete gene tokens. To address this mismatch, we adopt a state-transition-based masked discrete diffusion paradigm \cite{germain2015made}.

Formally, let $\mathbf{x}_0 = [x_1, \dots, x_L]$ denote the discrete gene sequence representing a cell state, where each gene token $x_i$ belongs to a finite vocabulary $\mathcal{V}$. We define the forward diffusion as a continuous-time Markov process $t \in [0, 1]$ that progressively destroys information. Unlike the local noise injection in continuous models, we employ a global stochastic corruption mechanism based on an absorbing state. At any arbitrary time $t$, a token $x_t^i$ either retains its original state $x_0^i$ or transitions to the absorbing state $[\texttt{MASK}]$ (denoted as $\varnothing$) with probability defined by:
\begin{equation}
q(x_t^i | x_0^i) = (1 - t) \cdot \delta(x_t^i, x_0^i) + t \cdot \delta(x_t^i, \varnothing).
\end{equation}
This process creates a smooth trajectory from a fully determined biological profile at $t=0$ to a state of maximum entropy at $t=1$. This formulation closely mirrors the ``dropout" phenomenon in single-cell sequencing \cite{hicks2018missing}, where valid signals are stochastically lost, thereby grounding our mathematical noise model in the underlying physical properties of the data (see Appendix \ref{app:absorbing_dropout} and \ref{app:dropout_isomorphism} for the proof of dropout isomorphism).



The generative capability of scDiVa is realized through the reverse denoising process, which learns the conditional distribution $p_\theta(\mathbf{x}_0 \mid \mathbf{x}_t)$. By reconstructing original gene states from masked inputs, scDiVa performs bidirectional, non-causal modeling \cite{devlin2019bert}, enabling each gene to be inferred from global context. This design avoids the artificial ordering and asymmetric dependencies imposed by autoregressive factorizations (e.g., $\prod p(x_i \mid x_{<i})$), which are biologically implausible given the non-sequential and symmetric nature of gene regulatory interactions \cite{levine2005gene}. As a result, scDiVa directly models the joint distribution $p(\mathbf{x})$ in a manner aligned with the unordered, multiset structure of gene expression profiles \cite{hao2024large}.
A formal comparison with autoregressive and Gaussian diffusion models is provided in Appendix~\ref{app:compare_models}.

\begin{figure*}[t]
\begin{center}
\centerline{\includegraphics[width=\textwidth]{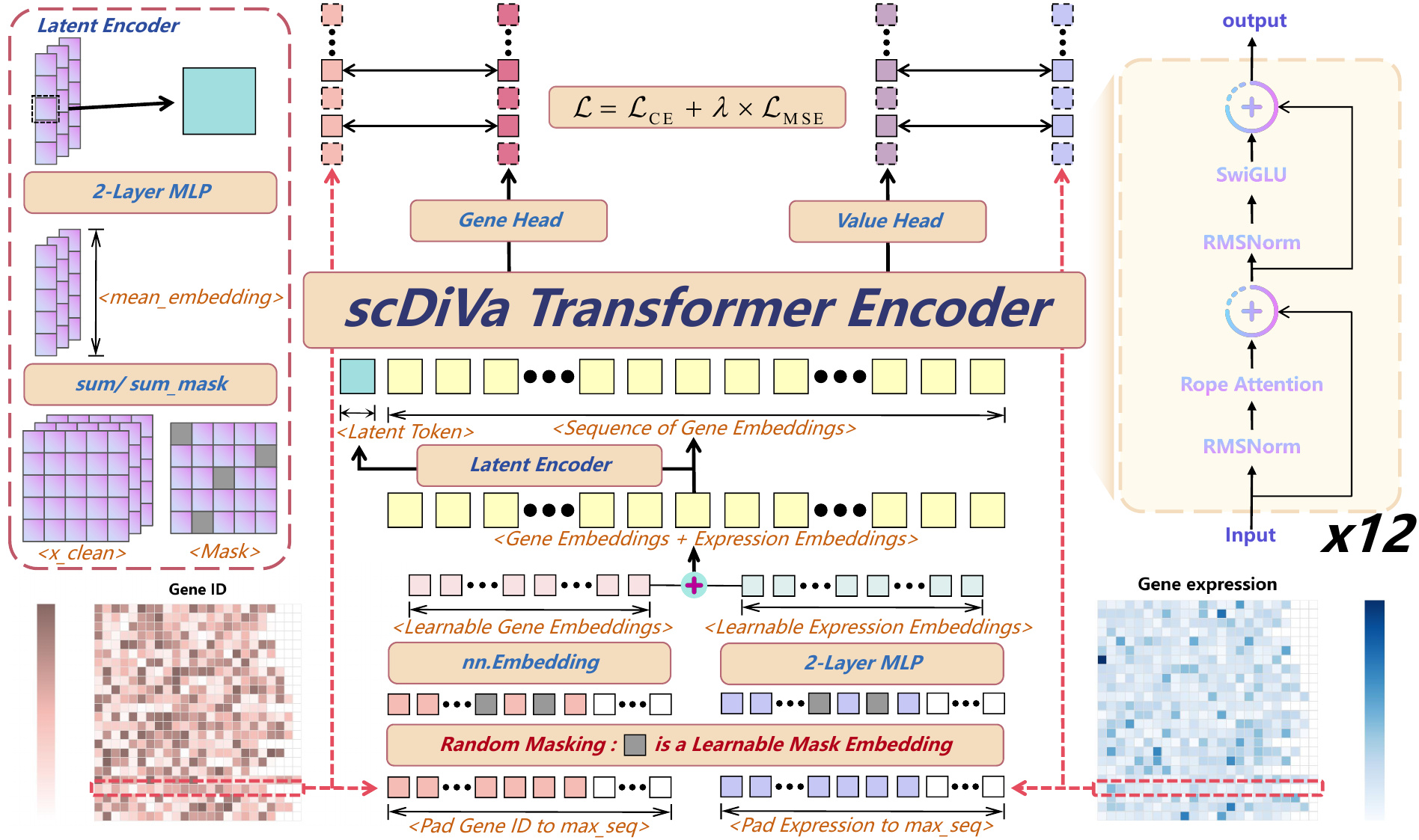}}
\caption{\textbf{Overview of the scDiVa Architecture.} The framework employs a masked modeling approach with a Latent Encoder to capture global cell contexts. The input gene expression profile is randomly masked and processed through a 12-layer Transformer encoder equipped with RoPE attention and SwiGLU activation. The model optimizes a dual objective ($\mathcal{L}$), combining Cross-Entropy ($\mathcal{L}_{CE}$) for gene identity reconstruction and Mean Squared Error ($\mathcal{L}_{MSE}$) for expression value regression.}
\label{fig:architecture}
\end{center}
\vskip -0.29in
\end{figure*}

\subsection{Cell Representation and Unified Embedding}
\label{subsec:cell_representation}

The adaptation of transcriptome data to Transformer architectures is non-trivial due to the inherent sparsity and long-tail distribution of gene expression. To maximize effective supervision within a finite context window $L$, we propose a specialized representation protocol that prioritizes information density over raw throughput.

Standard ``ranking by expression" approaches are suboptimal because they are often dominated by high-expression but low-entropy ``housekeeping" genes, which contribute minimal discriminative information regarding cell identity. To counter this, we implement \emph{Entropy-Normalized Serialization}. We define a ranking score $r_k$ that penalizes ubiquitous expression features using population-level Shannon entropy $H(g)$ \cite{brennecke2013accounting}:
\begin{equation}
r_k = \frac{v_k}{H(g_k) + \epsilon}.
\end{equation}
This formulation introduces a data-centric inductive bias, effectively filtering out biological background noise. It compels the model to allocate its computational budget to genes with the highest discriminative power \cite{stuart2019comprehensive}, ensuring that the limited token space encodes the maximum possible cell-type-specific information (see Appendix \ref{app:entropy_serialization} for detailed motivation).

Following feature selection, we construct a Unified Embedding that respects the unordered nature of the data. Since absolute positions in a gene sequence are biologically meaningless, our input embedding $\mathbf{h}^{(0)}$ integrates only gene identity and expression magnitude:
\begin{equation}
\mathbf{h}^{(0)} = \text{Emb}{gene}(g) + \text{MLP}{val}(v).
\end{equation}
While we discard absolute positional encodings to avoid imposing an artificial rigid Cartesian structure, capturing the intrinsic relationships between genes remains crucial. We therefore strategically inject relative position information via Rotary Positional Embedding (RoPE) \cite{su2024roformer} in the attention layers. This allows scDiVa to robustly learn hierarchical gene-gene dependencies without overfitting to a spurious permuted sequence order (further discussion on Transformer inductive bias is provided in Appendix \ref{app:transformer_bias}).

To further stabilize the learning process, particularly in the high-noise regimes ($t \to 1$) where the input signal is sparse, we introduce a \emph{Latent Variable Anchor Token} ($\texttt{[LAT]}$). In the encoder, $\texttt{[LAT]}$ aggregates global information via self-attention to form a state vector $\mathbf{z}_{lat}$ \cite{wang2021not}. During the reverse denoising process, $\mathbf{z}_{lat}$ functions as a ``Global Prompt", anchoring the generation to the underlying cell-state manifold. This prevents ``posterior collapse" often observed in conditional generation \cite{higgins2017beta}, enabling scDiVa to maintain identity coherence even when 90\% the gene tokens are masked (see Appendix \ref{app:latent_anchor}).

This architecture offers distinct theoretical advantages over traditional AR paradigms. First, by avoiding causal masking, scDiVa achieves Permutation Invariance, eliminating the false order sensitivity (e.g., $g_A \to g_B$) inherent in sequential models. Second, the parallel denoising mechanism prevents Exposure Bias \cite{bengio2015scheduled}, ensuring that errors do not cascade linearly across the sequence. Finally, the model facilitates Omnidirectional Flow, capturing complex feedback loops and regulatory logic that unidirectional attention mechanisms fail to represent.
\subsection{Depth-Invariant Sampling and Optimization}
\label{subsec:optimization}

We conceptualize the training process not merely as a denoising task, but as an inverse simulation of the physical sequencing process. By modeling the stochastic degradation of mRNA, scDiVa learns a universal operator capable of bridging the quality gap between low-depth droplet data and high-fidelity full-length transcriptomes.

To achieve this, we employ a Depth-Invariant Sampling strategy. Unlike BERT-style models that rely on a fixed masking rate, we sample the diffusion time step $t \sim \mathcal{U}(0,1)$ continuously \cite{dieleman2022continuous}. Physically, $t$ acts as a proxy for the reciprocal of sequencing depth, simulating the full spectrum of data sparsity. This compels the model to learn \emph{Depth Invariance}, effectively mapping inputs of varying sparsity onto a shared, canonical latent manifold \cite{gayoso2022python}. We elaborate on the correspondence between diffusion time and physical sequencing depth in Appendix \ref{app:depth_sampling}.
The optimization objective is governed by a Dual Denoising Loss, which imposes simultaneous constraints on the masked set $\mathcal{M}$ to recover both the network topology and the precise cellular state. The loss function $\mathcal{L}$ is formulated as:
\begin{equation}
\mathcal{L} = \mathbb{E}_{t, \mathbf{x}_0} \left[ \sum_{i \in \mathcal{M}} \underbrace{-\log p_\theta(g_i \mid \mathbf{x}_t)}_{\mathcal{L}_{\text{id}}} + \lambda \underbrace{\| \hat{v}_i - v_i \|^2}_{\mathcal{L}_{\text{val}}} \right].
\end{equation}
Here, the classification term $\mathcal{L}_{id}$ reconstructs the probabilistic topology of the Gene Regulatory Network (GRN), ensuring the correct identification of active genes. Concurrently, the regression term $\mathcal{L}_{val}$ enforces accurate inference of expression dosages \cite{eraslan2019single}.We provide the rigorous derivation showing that this objective optimizes a Variational Lower Bound (ELBO) on the data likelihood in Appendix \ref{app:elbo}. Minimizing this joint objective ensures that scDiVa learns a holistic representation of the cell that is robust to the technical variations inherent in different sequencing technologies.

\subsection{Model Architecture and Pre-training Protocol}
\label{subsec:implementation}

ScDiVa is instantiated as a 12-layer bidirectional Transformer encoder with approximately 94.5M parameters and a comprehensive vocabulary of 41,818 gene tokens. The complete workflow is shown in Figure \ref{fig:architecture}. Feed-forward layers employ SwiGLU activations \cite{shazeer2020glu}, and Pre-RMSNorm \cite{zhang2019root} is applied throughout for enhanced numerical stability, with all accumulation performed in \texttt{float32} precision (detailed architecture configurations are listed in Tables \ref{tab:model_config} and \ref{tab:train_config} (Appendix \ref{app:arch}). The complete training and inference algorithms are provided in Algorithm \ref{alg:training} and \ref{alg:inference} (Appendix \ref{app:algorithms}).

The model is pre-trained on a corpus of 59,162,450 single-cell transcriptomes, sourced from a large proprietary dataset spanning diverse tissues. Following standard filtering and log-normalization, we apply entropy-normalized serialization to retain the top 1,200 genes.Training is conducted on four NVIDIA A100-SXM4-40GB GPUs using a global batch size of 768 via gradient accumulation \cite{rasley2020deepspeed} and proceeds for four epochs under the depth-invariant sampling regime. Additional details on dataset composition, entropy-normalization math, and preprocessing are reported in Appendix \ref{app:data}.

\vspace{-0.1mm}
\section{Experiments}
\label{sec:Experiments}

In this section, we present a comprehensive evaluation of scDiVa across four tasks of increasing complexity. We begin with rank-value reconstruction \cite{eraslan2019single} to validate intrinsic co-expression patterns, followed by multi-batch integration \cite{gayoso2022python} to assess robustness against technical noise. We then evaluate cell type annotation (both fine-tuning and zero-shot) to test discriminative power, and finally gene perturbation prediction \cite{lotfollahi2023predicting} to verify causal inference. Extensive benchmarking against state-of-the-art foundation models \cite{hao2024large} confirms scDiVa's ability to bridge high-fidelity generation with precise biological discrimination.

\begin{table}[t]
  \centering
  \setlength{\tabcolsep}{4pt} 
  \renewcommand{\arraystretch}{0.95}
  \caption{\textbf{Rank-Value Joint Reconstruction across datasets.} Lower L-Dist is better; higher BLEU and Spearman are better.}
  \label{tab:rank_value_recon}
  \resizebox{\columnwidth}{!}{%
  \begin{tabular}{llrrr}
    \toprule
    Dataset & Model & L-Dist $\downarrow$ & BLEU $\uparrow$ & Spearman $\uparrow$ \\
    \midrule
    \multirow{4}{*}{PBMC12k}
      & GeneMamba\_U & 430 & 0.532 & 0.469 \\
      & Geneformer   & 23  & 0.968 & 0.703 \\
      & GeneMamba    & 6   & \textbf{0.987} & 0.711 \\
    \rowcolor{myteal!10} & scDiVa        & \textbf{5} & \textbf{0.987} & \textbf{0.812} \\ 
    \midrule
    \multirow{4}{*}{Pancreas}
      & GeneMamba\_U & 370 & 0.524  & 0.461 \\
      & Geneformer   & 25  & 0.956  & 0.763 \\
      & GeneMamba    & \textbf{12} & \textbf{0.991} & 0.792 \\
    \rowcolor{myteal!10} & scDiVa        & 13  & 0.965 & \textbf{0.812} \\ 
    \midrule
    \multirow{4}{*}{Zheng68k}
      & GeneMamba\_U & 432 & 0.581  & 0.503 \\
      & Geneformer   & 25  & 0.937  & 0.901 \\
      & GeneMamba    & 11  & \textbf{0.996} & 0.980 \\
    \rowcolor{myteal!10} & scDiVa        & \textbf{9} & 0.992 & \textbf{0.994} \\ 
    \midrule
    \multirow{4}{*}{Immune}
      & GeneMamba\_U & 468 & 0.659  & 0.442 \\
      & Geneformer   & 17  & 0.962  & 0.823 \\
      & GeneMamba    & 12  & \textbf{0.998} & 0.844 \\
    \rowcolor{myteal!10} & scDiVa        & \textbf{4} & 0.997 & \textbf{0.970} \\ 
    \bottomrule
  \end{tabular}%
  }
\vskip -0.1in
\end{table}

\subsection{Downstream Tasks and Comparative Methods}
\label{subsec:datasets_baselines}

\definecolor{CustomBlue}{RGB}{30,144,255}     
\definecolor{CustomRed}{RGB}{255,64,64}       
\definecolor{CustomOrange}{RGB}{225,100,014}   
\definecolor{CustomGreen}{RGB}{65,172,74}     
\definecolor{CustomPurple}{RGB}{214,77,237}  

\begin{table*}[t]
\centering
\renewcommand{\arraystretch}{1}
\caption{Benchmark results of the models on multi-batch experiments with BATCH and Cell metrics.}
\vskip -0.05in
\label{tab:multi_batch_results}
\resizebox{\textwidth}{!}{%
\begin{tabular}{llccccc>{\columncolor{myteal!10}}c} 
\toprule
\textbf{Metric} & \textbf{Dataset} & \textbf{Harmony} & \textbf{Geneformer} & \textbf{scGPT} & \textbf{scFoundation} & \textbf{GeneMamba} & \textbf{scDiVa} \\
\midrule
\multirow{5}{*}{\textbf{Avg\_batch}} 
& Immune            & 0.9514 & 0.8153 & 0.9194 & 0.8904 & 0.9536 & \textbf{0.9555} \\
& PBMC12k           & 0.9341 & 0.9545 & 0.9755 & 0.9628 & 0.9604 & \textbf{0.9960} \\
& BMMC              & 0.8999 & 0.7720 & 0.8431 & 0.7598 & 0.9157 & \textbf{0.9734} \\
& Perirhinal Cortex & 0.9442 & 0.9127 & 0.9600 & 0.9560 & \textbf{0.9673} & 0.9542 \\
& COVID-19          & 0.8781 & 0.8240 & 0.8625 & 0.8346 & 0.8742 & \textbf{0.9538} \\
\midrule
\multirow{5}{*}{\textbf{Avg\_bio}}    
& Immune            & 0.6945 & 0.6983 & 0.7879 & 0.7337 & \textbf{0.8131} & 0.7785 \\
& PBMC12k           & 0.7990 & 0.7891 & 0.9018 & 0.8662 & 0.8344 & \textbf{0.9566} \\
& BMMC              & 0.6316 & 0.6324 & 0.6576 & 0.5250 & 0.7628 & \textbf{0.8712} \\
& Perirhinal Cortex & 0.8595 & 0.8547 & 0.9552 & 0.9606 & 0.9062 & \textbf{0.9895} \\
& COVID-19          & 0.4468 & 0.5567 & 0.6476 & 0.5468 & 0.5537 & \textbf{0.6689} \\
\bottomrule
\end{tabular}
}
\vskip -0.1in
\end{table*}

\begin{figure}[t]
    \centering
    \includegraphics[width=\linewidth]{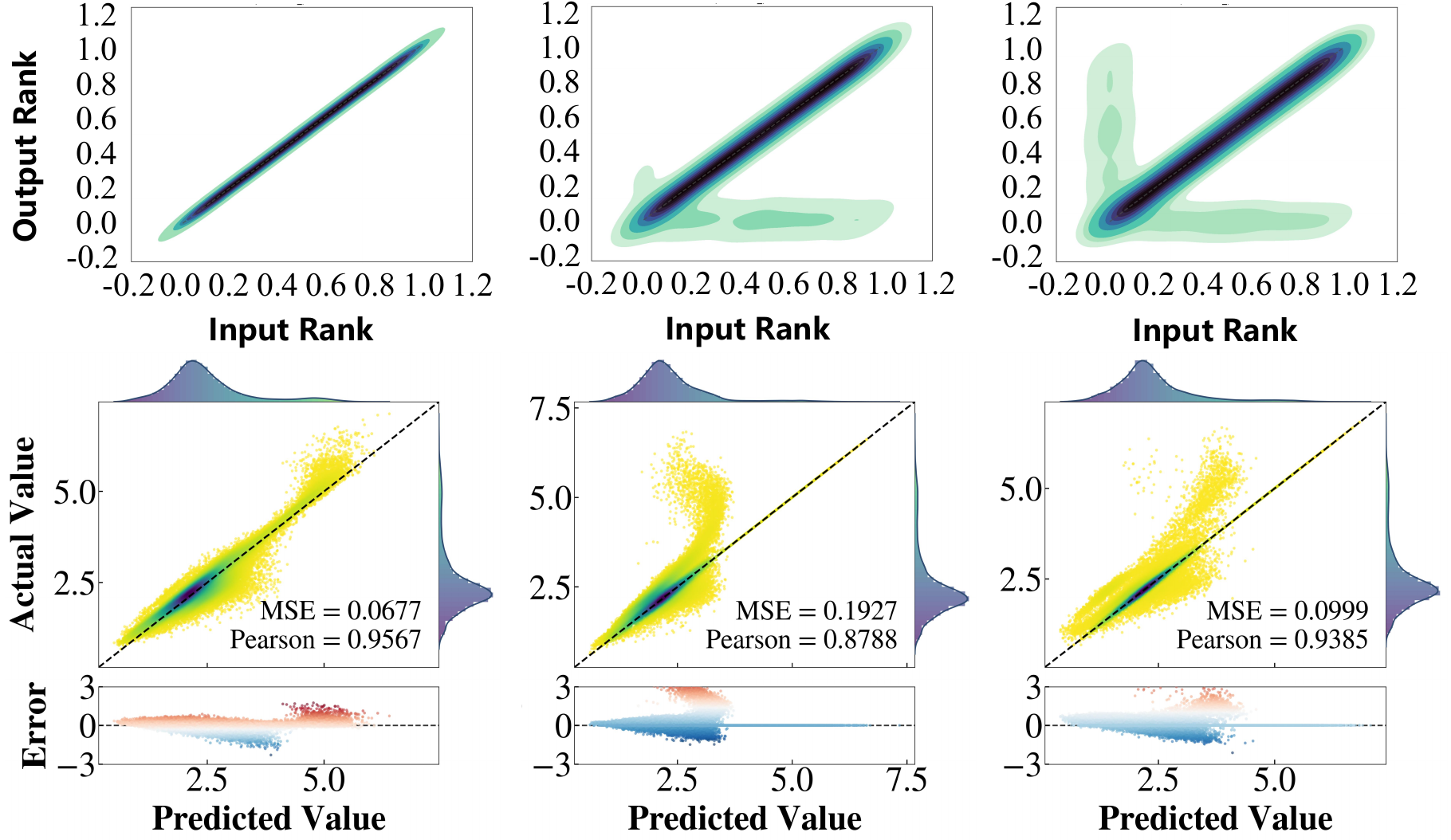}
    \vskip -0.05in
    \caption{\textbf{Visual comparison of generative dynamics.} From left to right: Zero-Mask Baseline (Ground Truth), Single-Step Inference, and 32-Step Diffusion Reconstruction.}
    \label{fig:reconstruction}
\vskip -0.18in
\end{figure}

\begin{figure*}[t]
\begin{center}
\centerline{\includegraphics[width=\textwidth]{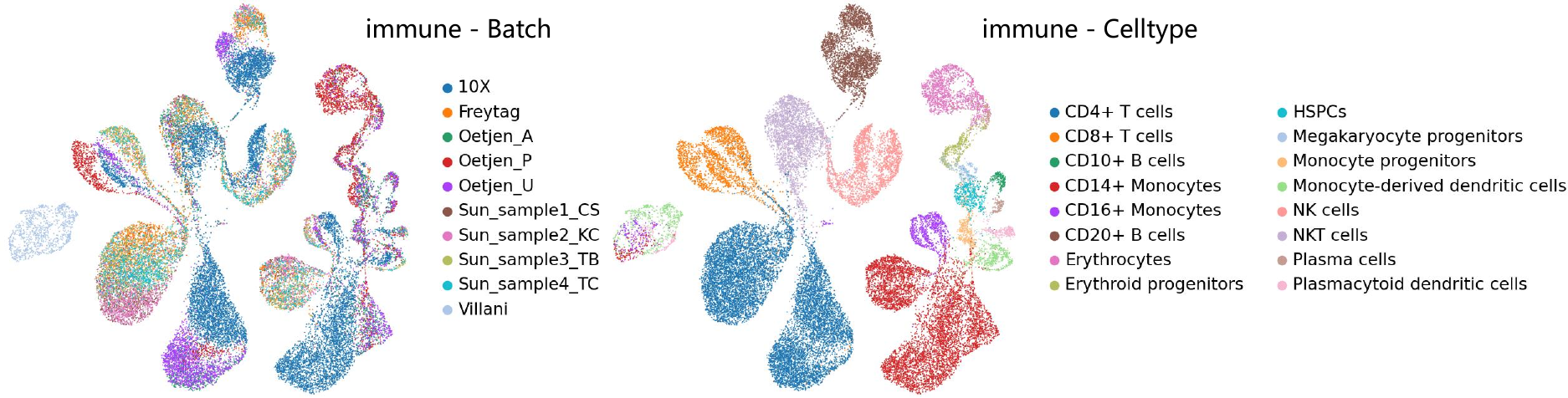}}
\vskip -0.05in
\caption{\textbf{Visualization of multi-batch integration on the Immune dataset.} The UMAP projections illustrate the latent representations learned by scDiVa. The left panel is colored by batch source, demonstrating the effective removal of batch effects (mixing). The right panel is colored by cell type, highlighting the distinct preservation of biological clusters and cellular identities.}
\label{fig:batch_immune}
\end{center}
\vskip -0.25in
\end{figure*}
To systematically evaluate scDiVa, we established a comprehensive benchmark encompassing diverse biological scenarios. Specifically, we utilized PBMC12k and Immune datasets for \textit{\textcolor{CustomBlue}{gene reconstruction}}. For \textit{\textcolor{CustomRed}{multi-batch integration}}, we assessed performance on COVID-19, Immune, PBMC12k, Perirhinal Cortex, and BMMC datasets. Adaptability in \textit{\textcolor{CustomOrange}{cell type annotation}} was evaluated via fine-tuning on hPancreas, MS, and Myeloid, alongside broad zero-shot evaluations on Cell\_Lines, DC, MCA, PBMC368K, HumanPBMC, and Pancrm. Furthermore, we employed Adamson and Norman datasets for \textit{\textcolor{CustomPurple}{gene perturbation prediction}}, and validated interpretability via \textit{\textcolor{CustomGreen}{gene correlation analysis}}.


We benchmarked scDiVa against SOTA foundation models and task-specific algorithms. Comparisons included Transformer architectures like Geneformer \cite{theodoris2023transfer} and the integration toolkit Harmony \cite{korsunsky2019fast}. Zero-shot baselines extended to CellFM \cite{zeng2025cellfm}, scBERT \cite{yang2022scbert}, GeneCompass \cite{yang2024genecompass}, UCE \cite{rosen2023universal}, as well as SVM and scmap \cite{abdelaal2019comparison}. For perturbation, we compared against specialized causal methods such as GEARS \cite{roohani2024predicting}. All evaluations adhered to preprocessing pipelines strictly aligned with original literature. Detailed dataset statistics and metric definitions are provided in Table \ref{tab:downstream_stats} (Appendix \ref{app:downstream_table}) and Appendix \ref{app:metrics}.


\begin{figure*}[t]
    \centering
    \includegraphics[width=\linewidth]{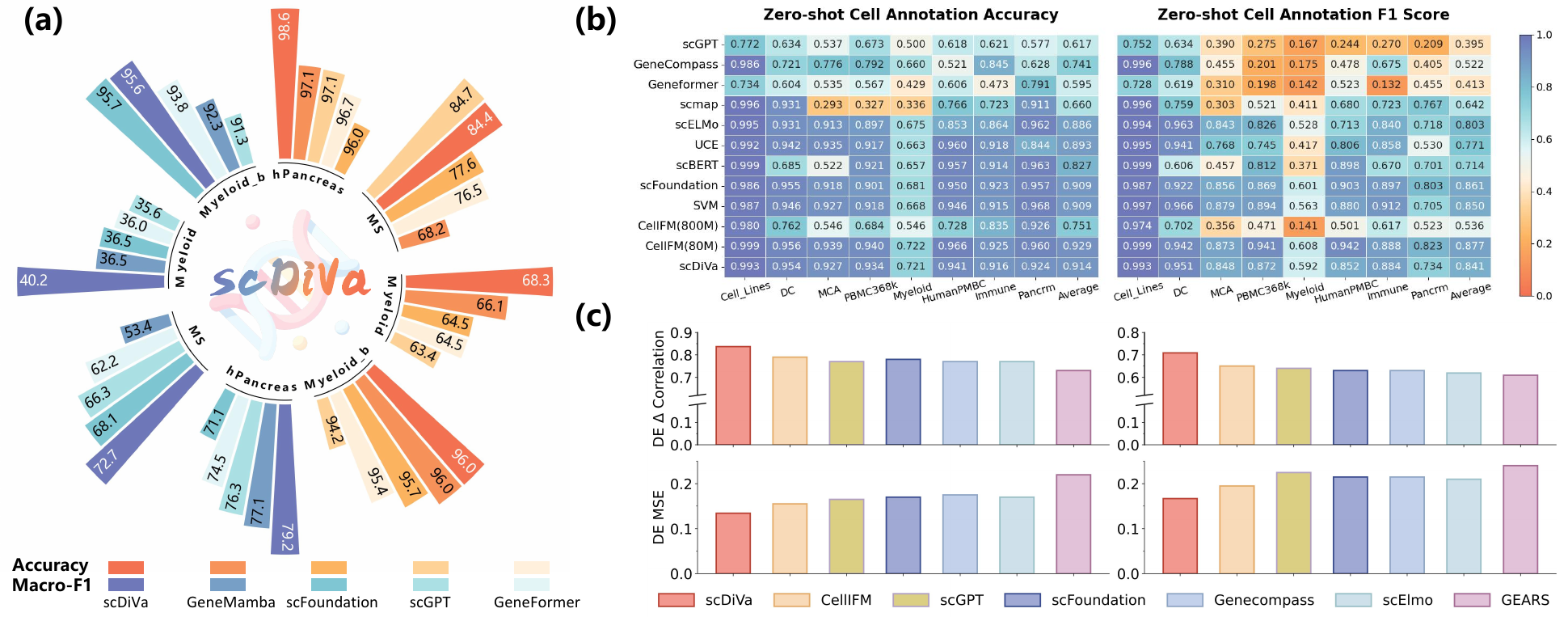}
    \vskip -0.05in
    \caption{\textbf{Comprehensive evaluation of scDiVa performance.} (a) Cross-batch full fine-tuning performance (Accuracy and Macro-F1) on hPancreas, MS, Myeloid, and Myeloid\_b datasets. (b) Zero-shot cell annotation performance (Accuracy and F1 Score) achieved by freezing the backbone and training only the MLP head across various datasets. (c) Perturbation prediction comparison against other models on the Adamson (left) and Norman (right) datasets, evaluated using DE $\Delta$ Correlation and DE MSE metrics.}
    \label{fig:组图1}
\end{figure*}
\begin{figure}[t]
    \vskip -0.1in
    \centering
    \includegraphics[width=\linewidth]{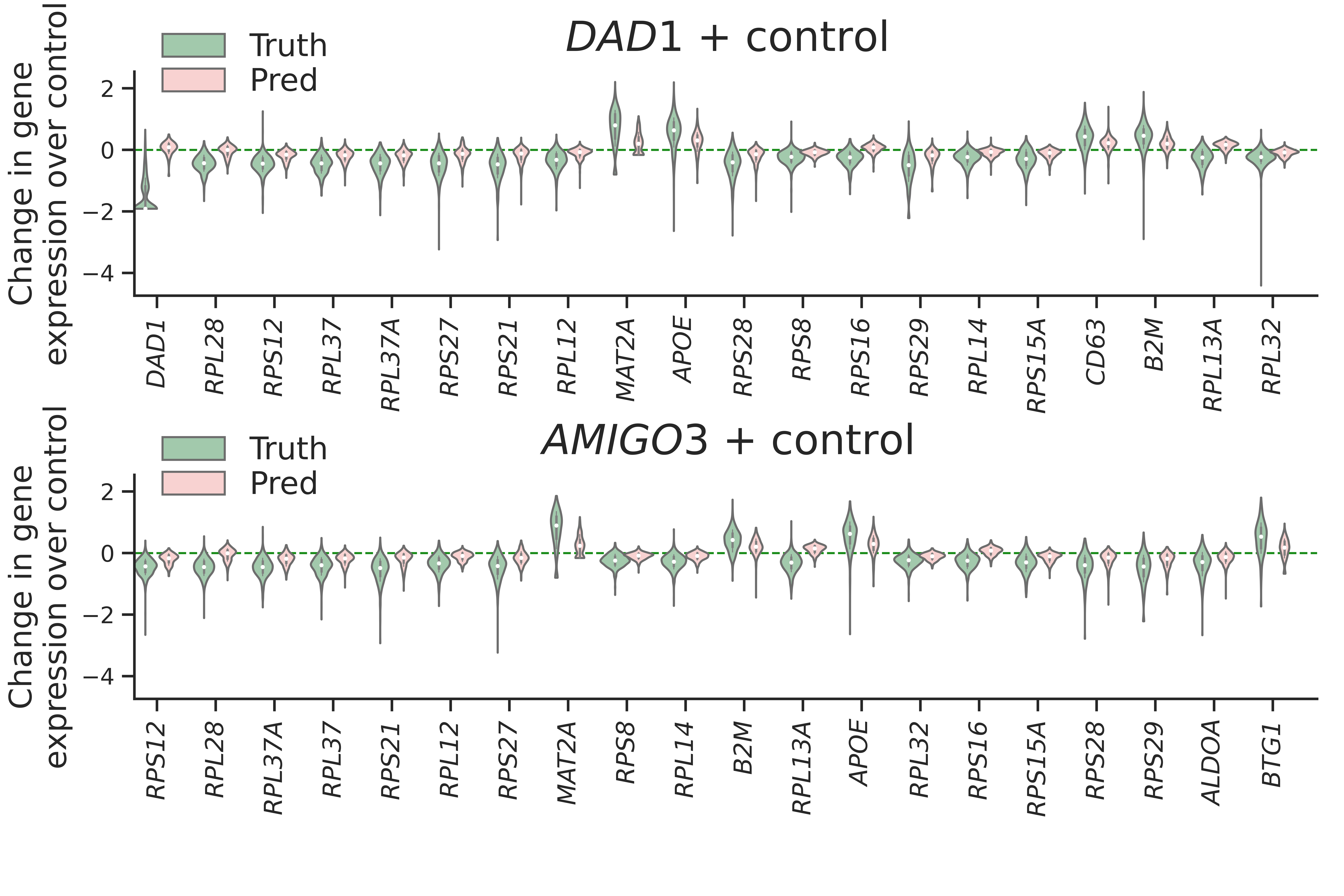}
    \vskip -0.1in
    \caption{\textbf{Predicted vs. observed expression shifts in the Adamson dataset.} Distributional changes for the top 20 differentially expressed genes under DAD1 and AMIGO3 perturbations.}
    \label{fig:panel_b_adamson}
\vskip -0.1in
\end{figure}

\vspace{-0.1mm}
\subsection{Rank-Value Joint Reconstruction}

Gene reconstruction requires restoring both relative Rank (robust against scaling noise, capturing underlying topology) and absolute Value (enforcing magnitude constraints). This dual focus effectively prevents ``pseudo-consistency'' where correct ranking masks intensity collapse.

We evaluate generative capability via a novel mask-free strategy, reordering sampled gene-value pairs based on model probabilities. Performance is comprehensively assessed using L-Dist (distribution shifts), BLEU (sequence matching), and Spearman coefficient (rank correlation).

Table \ref{tab:rank_value_recon} highlights model performance across four datasets. GeneMamba\_U performs poorly, lacking fine-grained constraints. While Geneformer and GeneMamba achieve high BLEU scores (0.94--0.998), they fail to fully capture biological rank structures. In contrast, scDiVa excels in Rank while maintaining high BLEU, achieving record Spearman correlations on Immune \textbf{(0.9701, +14.9\%)} and PBMC12k \textbf{(0.812, +14.2\%)}. This confirms scDiVa successfully recovers latent global ranking dependencies. Further analysis on the intersection of reconstructed gene features across datasets is presented in Figure \ref{fig:app_recon} in Appendix \ref{app:gene_recon}.

To investigate robustness, Figure \ref{fig:reconstruction} compares One-step Inference with Multi-step Diffusion. While the backbone handles mask-free settings well, a 30\% mask induces noticeable slight distribution shifts in One-step Inference, with local Rank deviations and heavy-tailed residuals indicating limitations in handling uncertainty. Conversely, multi-step diffusion acts as an iterative denoising mechanism, pulling samples toward low-energy manifolds to improve Rank convergence and reduce residual variance. However, given the significant latency cost for primarily local gains, we prioritize computational efficiency. As One-step Inference provides sufficient fidelity, we adopt it as the default strategy for subsequent complex downstream biological tasks.

\vspace{-0.1mm}
\subsection{Multi-batch Integration} 

We systematically evaluated scDiVa on complex multi-batch integration tasks, aiming to harmonize datasets from distinct batches while eliminating technical effects and preserving intrinsic biological heterogeneity.

To address this, we fine-tuned scDiVa using adversarial domain adaptation and a Supervised Contrastive (SupCon) loss, integrating a Gradient Reversal Layer (GRL) and batch discriminator. Fine-tuning was performed across five benchmarks: Immune, PBMC12k, BMMC, Perirhinal Cortex, and COVID-19, with performance evaluated via deep latent embeddings (Figure \ref{fig:batch_immune}; see Figure \ref{fig:app_batch} in Appendix \ref{app:batch_viz} for visualizations of the other four datasets).

For rigorous comparison, we used Avg-batch to assess batch correction and Avg-bio to evaluate biological conservation, metrics that typically reflect a fundamental trade-off between noise removal and feature preservation.

Table \ref{tab:multi_batch_results} demonstrates that scDiVa effectively reconciles this trade-off. While Harmony removes batch effects, it sacrifices biological detail (e.g., 0.4468 Avg-bio on COVID-19). In contrast, scDiVa matches or surpasses state-of-the-art models in Avg-batch (0.9960 on PBMC12k vs. 0.9755 for scGPT) while exhibiting a dominant advantage in Avg-bio. Notably, it achieves 0.8712 on BMMC, significantly exceeding GeneMamba's 0.7628. These results indicate that scDiVa successfully disentangles technical noise via adversarial learning while leveraging pre-trained knowledge to reconstruct biological topology, demonstrating superior robustness in complex integration scenarios.

\subsection{Cell Type Annotation}
We systematically evaluated cellular representation via two distinct paradigms: fine-tuning to assess cross-batch adaptation, and zero-shot evaluation to validate intrinsic linear separability by training only the MLP head.

Fine-tuning demonstrates robust stability and superior resolution for long-tail classes. In cross-batch settings (Figure \ref{fig:组图1}a), scDiVa excels in the imbalance-sensitive Macro-F1. On hPancreas, it achieves 0.986 accuracy and 0.7919 Macro-F1, indicating discriminative representations. Detailed confusion matrices and hierarchical clustering heatmaps for these fine-tuning tasks are shown in Figure \ref{fig:app_annot} (Appendix \ref{app:cell_annot}). Notably, on the highly imbalanced MS dataset, scDiVa reaches a Macro-F1 of 0.7271, 36\% improvement over GeneMamba (0.5342). These gains stem not from overfitting majority classes, but from effectively disentangling batch noise to improve decision boundaries for rare populations.

Zero-shot evaluation across eight datasets (Figure \ref{fig:组图1}b) confirms generalized semantic alignment. scDiVa achieves an average accuracy of 0.914 and Macro-F1 of 0.841, significantly performing Transformer baselines like scGPT and Geneformer. While competitive with specialized models like CellFM, scDiVa's design strategically balances generalization and plasticity. Combined with fine-tuning results, this suggests the latent space avoids overfitting specific zero-shot distributions, providing reliable initial classification while retaining superior adaptation potential for downstream tasks with minimal supervision.

\begin{figure}[t]
    \centering
    \includegraphics[width=\linewidth]{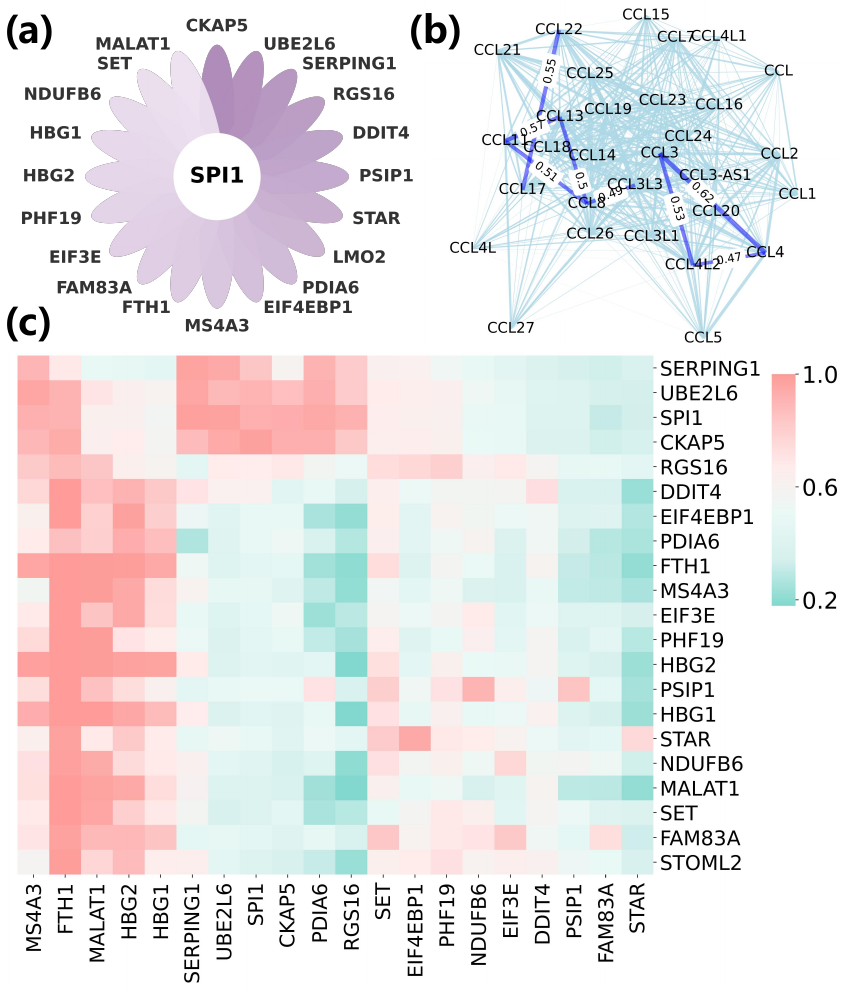}
    \vskip -0.05in
    \caption{\textbf{Visualization of Gene Regulatory Networks.} (a) Global GRN constructed from scDiVa embeddings. (b) The inferred regulon of the master regulator SPI1. (c) Attention heatmap illustrating the regulatory hub structure of SPI1.}
    \label{fig:GRN}
\vskip -0.2in
\end{figure}

\subsection{Perturbation Prediction}

The central challenge in gene perturbation prediction lies in bridging the causal gap between observational and interventional distributions, a task that necessitates inferring transcriptional cascades in the absence of direct samples. To address this, we employ a fine-tuning strategy on a pre-trained backbone. To mitigate the inherent sparsity of single-cell data, we introduce a Signal-sensitive Weighted Mean Squared Error (MSE) as the loss function, which significantly enhances the model's sensitivity to biological signals.

Experiments were conducted on two representative benchmark datasets. The Adamson dataset, comprising 87 single-gene perturbations, primarily evaluates the inference of single causal links. Conversely, the Norman dataset includes 131 double-gene perturbations and serves to probe the model's capacity to resolve non-additive genetic interactions. We benchmarked scDiVa against leading methods, including CellFM, scGPT, and GEARS.

As illustrated in Figure~\ref{fig:组图1}c, scDiVa demonstrates SOTA performance across both tasks. On the Adamson dataset, scDiVa achieves a Pearson correlation of 0.837 and a Mean Squared Error of 0.134 for differentially expressed genes. These results outperform baseline models, confirming the precision of our approach in single-point perturbation inference (Figure~\ref{fig:panel_b_adamson}). Additionally, scDiVa achieves superior Hit Rate in retrieving top-$k$ responsive genes, as analyzed in Figure \ref{fig:app_hitrate} (Appendix \ref{app:perturb}). Notably, this superiority extends to the dual-perturbation tasks within the Norman dataset, where scDiVa yields a correlation of 0.709. This generalization capability indicates that the model has effectively learned complex combinatorial logic underlying gene regulatory networks rather than memorizing training samples.

\subsection{Gene Regulatory Network Inference and Logic}
To validate interpretability, we constructed a comprehensive global Gene Regulatory Network (GRN) from scDiVa embeddings (Figure \ref{fig:GRN}a), revealing distinct functional clusters. Extended visualizations of inferred regulatory networks for other critical families (TNF, Interleukins, CD markers) are provided in Figure \ref{fig:app_grn} (Appendix \ref{app:grn_viz}). Using attention-based perturbation analysis on the myeloid master regulator SPI1, we identified the top-20 responsive genes via rank-normalized weights, successfully recovering a biologically coherent regulon (Figure \ref{fig:GRN}b).

Notably, the model precisely captures opposing forces in cell fate determination. It assigns high attention to hemoglobin genes (HBG1/2), reflecting SPI1's role in repressing erythroid differentiation to lock in myeloid fate. Concurrently, it identifies myeloid markers MS4A3 and macrophage-critical FTH1. This confirms scDiVa reconstructs SPI1's complete logic: promoting myeloid development while actively repressing erythroid potential. Crucially, generic cell-cycle genes (e.g., CCNB2, TOP2A) are absent despite typical co-expression, demonstrating scDiVa's ability to effectively filter non-causal background noise.

The heatmap (Figure \ref{fig:GRN}c) reveals structural asymmetry, with SPI1 showing strong vertical connectivity, characterizing it as a regulatory hub. A dense cluster linking SPI1, SERPING1, UBE2L6, and CKAP5 indicates high-order coupling of immune defense (ISG15 pathway) with cytoskeletal remodeling. This suggests scDiVa encodes differentiation as a coordinated program of simultaneous morphological and immune acquisition rather than isolated events. 

\vspace{-5pt}
\section{Conclusion}

In this work, we present scDiVa, a foundation model for single-cell representation learning based on masked discrete diffusion. By aligning the generative process with sequencing dropout and adopting a dual denoising objective, scDiVa mitigates the structural biases of autoregressive models and enables accurate reconstruction of both gene identity and expression magnitude. Extensive evaluations across integration, annotation, and perturbation tasks demonstrate that modeling gene expression as an unordered stochastic multiset improves biological fidelity. These results highlight discrete diffusion as a principled alternative to sequential paradigms for large-scale single-cell modeling.

\section*{Impact Statement}

This paper presents a foundation model, scDiVa, which aims to advance the field of computational biology and single-cell genomics. The goal of this work is to improve the precision of cellular representation learning, which has potential benefits for understanding disease mechanisms, drug target discovery, and personalized medicine.

There are many potential societal consequences of our work, none which we feel must be highlighted here. However, we acknowledge that the training of large-scale foundation models consumes significant computational resources and energy. We have made efforts to optimize our training strategy (e.g., Depth-Invariant Sampling) to improve efficiency. Furthermore, while the model is trained on public anonymized data, future applications in clinical settings must adhere to patient privacy and ethical data usage standards.


\bibliography{example_paper}
\bibliographystyle{icml2026}

\clearpage
\appendix

\section{Additional Methodological Details}

\subsection{Biological Interpretation of Masked Discrete Diffusion}
\label{app:bio_diffusion}

Single-cell RNA sequencing measurements can be viewed as noisy observations of an underlying biological state, where gene expression signals are subject to stochastic loss due to limited capture efficiency, amplification bias, and sequencing depth constraints. In this context, the masked discrete diffusion process employed by scDiVa provides a natural abstraction of technical dropout. 

Unlike continuous Gaussian noise, which perturbs values symmetrically in Euclidean space, dropout events correspond to the complete disappearance of gene-level signals. The absorbing \texttt{[MASK]} state used in scDiVa explicitly models this irreversible information loss. As diffusion time increases, the probability of observing a true gene signal decreases monotonically, reflecting progressively lower effective sequencing depth. This correspondence grounds the forward diffusion process in the physical data-generating mechanism of single-cell experiments.

\subsection{Absorbing-State Corruption and Dropout Correspondence}
\label{app:absorbing_dropout}

The forward corruption process in scDiVa independently replaces each gene token with an absorbing \texttt{[MASK]} state with probability proportional to diffusion time. This design differs fundamentally from additive noise models, where corrupted values remain informative through magnitude. In contrast, masked corruption enforces a discrete transition between observable and unobservable states, matching the binary nature of dropout events.

Importantly, the absorbing state ensures that once a gene is masked, no residual information about its original identity or magnitude remains in the corrupted representation. This property allows the reverse denoising process to learn explicit conditional dependencies among genes, rather than relying on smooth interpolation in value space.

\subsection{Comparison with Autoregressive and Gaussian Diffusion Models}
\label{app:compare_models}

Autoregressive (AR) models factorize the joint distribution over genes into a product of conditional distributions ordered by an arbitrary sequence. This ordering introduces artificial causal dependencies that are biologically implausible for gene regulatory networks, where interactions are typically symmetric or cyclic. Moreover, AR generation is susceptible to exposure bias, as early prediction errors propagate irreversibly.

Gaussian diffusion models avoid explicit ordering but impose an ordinal structure on gene expression values, assuming that small perturbations preserve semantic meaning. This assumption is problematic for discrete gene activation events and sparse count distributions. By contrast, masked discrete diffusion directly models presence or absence through state transitions and performs reconstruction using full bidirectional context, avoiding both ordering bias and ordinal assumptions.

\subsection{Motivation for Entropy-Normalized Serialization}
\label{app:entropy_serialization}

Gene expression distributions exhibit extreme imbalance: a small number of housekeeping genes are ubiquitously expressed across cell types, while many low-frequency genes carry high discriminative value. Ranking genes solely by expression magnitude therefore biases token allocation toward high-abundance but low-information features.

Entropy-normalized serialization mitigates this effect by weighting expression values with inverse population-level entropy. Genes that are consistently expressed across cells receive lower priority, while genes with high cell-type specificity are emphasized. This strategy allows a finite token budget to encode maximal discriminative information and improves downstream transfer performance under fixed context length constraints.

\subsection{Inductive Bias of Transformer Parameterization}
\label{app:transformer_bias}

ScDiVa employs a bidirectional Transformer encoder to parameterize the reverse denoising distribution. Absolute positional encodings are omitted because gene order in the input sequence is not biologically meaningful. Introducing absolute positions would impose a Cartesian structure that is invariant neither to permutation nor to feature selection.

Relative positional encoding via RoPE is retained to allow the model to learn structured dependencies among co-occurring genes within the serialized sequence. Since serialization is deterministic given the ranking rule, relative positions encode comparative importance rather than temporal order. This design balances permutation invariance with sufficient inductive bias to learn hierarchical gene–gene interactions.

\subsection{Latent Anchor Token and High-Mask Stability}
\label{app:latent_anchor}

At high diffusion times, a large fraction of gene tokens are masked, resulting in sparse and unstable inputs. To address this, scDiVa introduces a latent anchor token \texttt{[LAT]} that participates in self-attention but is never masked. This token aggregates global information from observed genes and serves as a persistent conditioning signal during denoising.

Empirically, the latent anchor improves training stability and prevents identity drift when the visible token set is small. Conceptually, \texttt{[LAT]} functions as a global summary of the partially observed cell state, enabling consistent reconstruction even under extreme corruption.

\subsection{Depth-Invariant Sampling as Sequencing Depth Simulation}
\label{app:depth_sampling}

Uniform sampling of diffusion time exposes the model to varying corruption levels, which can be interpreted as simulating a continuum of effective sequencing depths. Low diffusion times correspond to high-coverage profiles, while high diffusion times emulate sparse droplet-based measurements.

By training across this spectrum, scDiVa learns a depth-invariant mapping from corrupted observations to a canonical latent representation. This eliminates the need for explicit depth normalization or batch correction and enables robust transfer across datasets with heterogeneous coverage characteristics.

\subsection{Proof of Dropout Isomorphism}
\label{app:dropout_isomorphism}

\paragraph{Forward process.}
Let a serialized cell be represented by a length-$L$ sequence of discrete gene identity tokens
$x_0 = (x_0^1,\dots,x_0^L)$ with $x_0^i \in \mathcal{V}$, where $\mathcal{V}$ is a finite vocabulary of genes.
Let $\varnothing$ denote the absorbing \texttt{[MASK]} state.
The (continuous-time) masked discrete diffusion defines, for each position $i$ and time $t\in[0,1]$,
\begin{equation}
q(x_t^i \mid x_0^i)
= (1-t) \delta(x_t^i, x_0^i) + t \delta(x_t^i,\varnothing),
\label{eq:forward_mask_kernel}
\end{equation}
where $\delta(\cdot,\cdot)$ is the Kronecker delta on $\mathcal{V}\cup{\varnothing}$.
Assuming conditional independence across positions yields
$q(x_t\mid x_0) = \prod_{i=1}^L q(x_t^i\mid x_0^i)$.

\begin{theorem}[Dropout--Diffusion Isomorphism]
\label{thm:dropout_isomorphism}
Define an observation map $\phi:\mathcal{V}\cup{\varnothing}\rightarrow \mathcal{V}\cup{0}$ by
$\phi(g)=g$ for $g\in\mathcal{V}$ and $\phi(\varnothing)=0$ (interpreting \texttt{[MASK]} as complete information loss).
For each position $i$, define the observed ``signal'' random variable $y_t^i=\phi(x_t^i)$.
Then conditioned on $x_0^i$, $y_t^i$ follows a \emph{zero-inflated} (dropout) corruption:
\begin{equation}
q(y_t^i \mid x_0^i)
= (1-t) \delta(y_t^i, x_0^i) + t \delta(y_t^i, 0).
\label{eq:zero_inflation}
\end{equation}
In particular, as $t\to 1$, $q(y_t^i\mid x_0^i)\Rightarrow \delta(y_t^i,0)$, i.e., complete dropout.
\end{theorem}

\begin{proof}
By definition of $\phi$ and the forward kernel in Eq.~\eqref{eq:forward_mask_kernel},
\begin{align}
\mathbb{P}(y_t^i = x_0^i \mid x_0^i) &= \mathbb{P}(x_t^i = x_0^i \mid x_0^i) = 1-t, \\
\mathbb{P}(y_t^i = 0 \mid x_0^i) &= \mathbb{P}(x_t^i = \varnothing \mid x_0^i) = t.
\end{align}
Since $y_t^i$ takes values only in ${x_0^i,0}$ under $\phi$, the conditional distribution is exactly the mixture in
Eq.~\eqref{eq:zero_inflation}.
Taking $t\to 1$ gives $\mathbb{P}(y_t^i=0\mid x_0^i)\to 1$, hence weak convergence to $\delta(y_t^i,0)$.
\end{proof}

\paragraph{Connection to expression dropout.}
For a paired gene--value representation $(g_i,v_i)$, define the value corruption $\tilde v_t^i = \mathbb{I}[x_t^i\neq \varnothing]\cdot v_i$.
Then $\tilde v_t^i$ obeys the same zero-inflated form:
\begin{equation}
q(\tilde v_t^i \mid v_i)
= (1-t) \delta(\tilde v_t^i, v_i) + t \delta(\tilde v_t^i, 0),
\label{eq:value_zero_inflation}
\end{equation}
which matches the canonical abstraction of technical dropout as complete signal disappearance.

\subsection{ELBO Derivation for Masked Discrete Diffusion}
\label{app:elbo}

We provide a rigorous likelihood lower bound for the discrete identity component; the value regression term is then interpreted as a Gaussian log-likelihood term.

\paragraph{Discrete-time construction.}
Let $0=t_0<t_1<\cdots<t_K=1$ be a discretization.
Define the forward Markov chain
\begin{align}
q(x_{t_k}\mid x_{t_{k-1}}) &= \prod_{i=1}^L q(x_{t_k}^i\mid x_{t_{k-1}}^i), \nonumber \\
q(x_{t_k}^i=\varnothing \mid x_{t_{k-1}}^i\neq \varnothing) &= \frac{t_k-t_{k-1}}{1-t_{k-1}}, \quad q(\varnothing\mid \varnothing)=1,
\label{eq:forward_markov_discretized}
\end{align}
so that the marginal $q(x_{t_k}^i\mid x_0^i)$ equals Eq.~\eqref{eq:forward_mask_kernel} at $t=t_k$.
Let the prior at $t_K$ be the fully-masked absorbing state:
\begin{equation}
p(x_{t_K})=\delta\big(x_{t_K},\varnothing^L\big).
\label{eq:prior_fully_masked}
\end{equation}
Let the reverse model be a parametric family $p_\theta(x_{t_{k-1}}\mid x_{t_k})$ (implemented by a bidirectional Transformer conditioned on $x_{t_k}$ and $t_k$).

\paragraph{Model likelihood.}
The induced model distribution over clean sequences is
\begin{equation}
p_\theta(x_0) = \sum_{x_{t_1},\dots,x_{t_K}} p(x_{t_K}) \prod_{k=K}^{1} p_\theta(x_{t_{k-1}}\mid x_{t_k}),
\label{eq:model_marginal}
\end{equation}
where $x_{t_0}\equiv x_0$.

\paragraph{Variational lower bound (ELBO).}
Choose the variational posterior as the forward diffusion path
$q(x_{t_{1:K}}\mid x_0)=\prod_{k=1}^K q(x_{t_k}\mid x_{t_{k-1}})$.
Then,
\begin{align}
\log p_\theta(x_0) &= \log \mathbb{E}_{q(x_{t_{1:K}}|x_0)} \bigg[\frac{p(x_{t_K})\prod_{k=K}^{1} p_\theta(x_{t_{k-1}}\mid x_{t_k})}{q(x_{t_{1:K}}\mid x_0)} \bigg] \nonumber\\
&\ge \mathbb{E}_{q} \Bigg[ \log p(x_{t_K}) + \sum_{k=K}^1 \log \frac{p_\theta(x_{t_{k-1}}|x_{t_k})}{q(x_{t_k}|x_{t_{k-1}})} \Bigg] \nonumber\\
&\triangleq \mathcal{L}_{\mathrm{ELBO}}(\theta;x_0),
\label{eq:elbo_path}
\end{align}
where the inequality is Jensen's inequality.
By construction, $\mathcal{L}_{\mathrm{ELBO}}(\theta;x_0)\le \log p_\theta(x_0)$.

\paragraph{Optimization target.}
In Eq.~\eqref{eq:elbo_path}, the terms $\log p(x_{t_K})$ and $\log q(\cdot)$ do not depend on $\theta$.
Therefore, maximizing $\mathcal{L}_{\mathrm{ELBO}}$ is equivalent to maximizing
$\mathbb{E}_q\big[\sum_{k=K}^{1}\log p_\theta(x_{t_{k-1}}\mid x_{t_k})\big]$.
A standard Monte-Carlo estimator is obtained by sampling a random timestep $k$ (or continuous $t$) and training the model to predict the uncorrupted token identities at masked positions.

\paragraph{Single-step objective and the scDiVa loss.}
Let $t\sim \mathrm{Unif}(0,1)$ and $x_t\sim q(\cdot\mid x_0,t)$ as in Eq.~\eqref{eq:forward_mask_kernel}.
Let $M_t=\{i: x_t^i=\varnothing\}$ be the masked index set.
Assume the reverse conditional factorizes over positions given $x_t$ (standard in masked modeling):
\begin{align}
p_\theta(x_0\mid x_t,t) &= \prod_{i\in M_t} p_\theta(x_0^i\mid x_t,t), \label{eq:factorized_denoiser} \\
\log p_\theta(x_0\mid x_t,t) &= \sum_{i\in M_t}\log p_\theta(x_0^i\mid x_t,t). \label{eq:masked_loglik}
\end{align}
Using the practical normalization $|M_t|^{-1}$ (equivalently $\approx (tL)^{-1}$ in expectation) yields the training objective
\begin{equation}
\begin{split}
\mathcal{L}_{\mathrm{id}}(\theta) \triangleq \mathbb{E}_{x_0, t, x_t} \bigg[ \frac{1}{|M_t|}\sum_{i\in M_t}\log p_\theta(x_0^i\mid x_t,t) \bigg],
\end{split}
\label{eq:identity_objective}
\end{equation}
which is a (constant-shifted) stochastic estimator of the reverse term in Eq.~\eqref{eq:elbo_path}, and thus maximizes a lower bound on $\mathbb{E}_{p_{\mathrm{data}}}[\log p_\theta(x_0)]$.

\paragraph{Incorporating continuous values.}
For each token position $i$, scDiVa predicts both identity and value.
Let $g_i$ denote the true gene identity and $v_i\in\mathbb{R}$ its (log-normalized) expression.
We interpret the regression term as a Gaussian likelihood
$p_\theta(v_i\mid x_t,t)=\mathcal{N}(\hat v_i,\sigma^2)$ with fixed variance $\sigma^2$.
Then
\begin{equation}
\log p_\theta(v_i\mid x_t,t) = -\frac{1}{2\sigma^2}|\hat v_i-v_i|^2 + \mathrm{const}.
\label{eq:gaussian_value_ll}
\end{equation}
Consequently, defining $\lambda\triangleq\frac{1}{2\sigma^2}$, the joint objective
\begin{align}
\mathcal{L}_{\text{total}}(\theta) &\triangleq \mathcal{L}_{\mathrm{id}}(\theta) + \lambda \mathbb{E}\left[ \frac{1}{|M_t|}\sum_{i\in M_t}\log p_\theta(v_i\mid x_t,t) \right]
\end{align}
maximizes a lower bound on the joint log-likelihood of identities and values under the assumed factorization.

\subsection{Algorithms}
\label{app:algorithms}

\begin{algorithm}[H]
\caption{scDiVa Training}
\label{alg:training}
\begin{algorithmic}[1]
\REQUIRE Dataset $\mathcal{D}$ of cells; max length $L_{\max}$; model $p_\theta$ with \texttt{[LAT]}; loss weight $\lambda$
\STATE Initialize $\theta$ (AdamW optimizer)
\REPEAT
\STATE Sample mini-batch of cells ${(g^{(b)}_{1:L_b}, v^{(b)}_{1:L_b})}_{b=1}^B$ after entropy-normalized serialization
\STATE Pad to $L_{\max}$ using \texttt{[PAD]}; prepend \texttt{[LAT]} token (never masked)
\STATE Sample diffusion time $t \sim \mathrm{Unif}(0,1)$
\FOR{$b=1$ to $B$}
\STATE For each position $i$ (excluding \texttt{[LAT]} and \texttt{[PAD]}), sample mask $m_i \sim \mathrm{Bernoulli}(t)$
\STATE Set corrupted token $x_t^i \leftarrow \varnothing$ if $m_i=1$, else $x_t^i \leftarrow g_i$; record $M_t={i:m_i=1}$
\STATE Replace masked value input with a sentinel (e.g., $0$) or a learned \texttt{[MASK]} value embedding
\ENDFOR
\STATE Forward pass: obtain hidden states ${h_i}$ from the bidirectional Transformer
\STATE Predict gene logits $\hat y_{\mathrm{id}}^i$ and value predictions $\hat v_i$ for all positions
\STATE Compute identity loss on masked positions:
\[
\mathcal{L}_{\mathrm{id}} \leftarrow -\frac{1}{|M_t|}\sum_{i\in M_t}\log p_\theta(g_i \mid x_t,t)
\]
\STATE Compute value loss on masked positions:
\[
\mathcal{L}_{\mathrm{val}} \leftarrow \frac{1}{|M_t|}\sum_{i\in M_t}|\hat v_i - v_i|^2
\]
\STATE Total loss: $\mathcal{L}\leftarrow \mathcal{L}_{\mathrm{id}} + \lambda \mathcal{L}_{\mathrm{val}}$
\STATE Update parameters $\theta \leftarrow \mathrm{AdamW}(\theta,\nabla_\theta \mathcal{L})$
\UNTIL{convergence}
\end{algorithmic}
\end{algorithm}

\begin{algorithm}[t]
\caption{scDiVa Inference}
\label{alg:inference}
\begin{algorithmic}[1]
\REQUIRE Trained model $p_\theta$; length $L$; steps $K$; schedule $\{t_k\}_{k=0}^K$ with $t_0=0<t_1<\cdots<t_K=1$
\STATE Initialize $x_{t_K}\leftarrow \varnothing^L$ (fully masked), prepend \texttt{[LAT]}
\FOR{$k=K$ down to $1$}
    \STATE Run mask predictor to obtain $p_\theta(\cdot\mid x_{t_k},t_k)$ and value predictions $\hat v$
    \STATE For each masked position $i$, sample $\tilde g_i \sim p_\theta(\cdot\mid x_{t_k},t_k)$ and set provisional unmasking
    \STATE Determine target masking ratio $t_{k-1}$; unmask a fraction $(t_k-t_{k-1})/t_k$ of currently masked tokens
    \STATE (Optional) Low-confidence remasking: remask the fraction $t_{k-1}/t_k$ of tokens with lowest confidence to match $t_{k-1}$
    \STATE Set $x_{t_{k-1}}$ accordingly; keep \texttt{[LAT]} unchanged
\ENDFOR
\end{algorithmic}
\end{algorithm}

\section{Dataset and Preprocessing Details}
\label{app:data}

\subsection{Entropy-Normalized Serialization}
\label{app:entropy_serialization_math}

Let $v_{c,g}\in\mathbb{R}_{\ge 0}$ denote the (log-normalized) expression of gene $g$ in cell $c$.
To quantify population-level ubiquity, we compute Shannon entropy for each gene $g$.
Let $X_g$ be a discretized random variable obtained by binning $v_{c,g}$ across the pre-training corpus (e.g., via fixed-width or quantile bins),
with empirical probabilities $p_g(x)=\mathbb{P}(X_g=x)$.
Then the gene entropy is
\begin{equation}
H(g) = -\sum_{x} p_g(x)\log\big(p_g(x)\big).
\label{eq:gene_entropy}
\end{equation}
Given a cell-specific value $v_{c,g}$, we define the entropy-normalized ranking score
\begin{equation}
S_g(c) = \frac{v_{c,g}}{H(g)+\epsilon},
\label{eq:entropy_score}
\end{equation}
where $\epsilon>0$ prevents division by zero.
For each cell $c$, scDiVa selects the top-$L_{\max}$ genes by descending $S_g(c)$ and serializes them deterministically into a length-$L_{\max}$ sequence.

\subsection{Dataset Composition and Preprocessing Details}
\label{app:dataset_details}

The pre-training corpus comprises 59,162,450 single-cell transcriptomes aggregated from diverse tissues, conditions, and sequencing technologies. Cells with fewer than 200 detected genes were removed. Expression counts were log-normalized following standard preprocessing practice.

Entropy statistics were computed globally across the pre-training corpus and fixed prior to model training. For each cell, the top 1,200 genes ranked by entropy-normalized score were selected and serialized deterministically. This identical preprocessing pipeline was applied during pre-training and downstream evaluation to ensure distributional consistency.

\subsection{Pre-training Data Summary}
\label{app:pretrain_data}

The pre-training corpus consists of a large-scale, proprietary single-cell transcriptomic dataset aggregated from internal sources. Due to strict data privacy regulations and commercial confidentiality agreements, the specific composition, donor metadata, and source breakdown of this corpus cannot be publicly disclosed. However, the dataset is curated to ensure high diversity, covering a wide range of tissue types, developmental stages, and sequencing technologies comparable to major public archives.

\subsection{Downstream Dataset Statistics}
\label{app:downstream_table}

For all downstream evaluation tasks, including fine-tuning, zero-shot learning, and integration, we exclusively utilized publicly available, open-source datasets. These benchmarks (summarized in Table \ref{tab:downstream_stats}) represent standard community resources, ensuring the transparency and reproducibility of our evaluation metrics.

\begin{table*}[t]
\centering
\caption{Downstream dataset statistics used in our evaluations. Sparsity is the fraction of zero entries in the raw gene-by-cell count matrix.}
\label{tab:downstream_stats}
\resizebox{\textwidth}{!}{
\begin{tabular}{l l r r r c r}
\toprule
\textbf{Dataset} & \textbf{Task} & \textbf{N\_Cells} & \textbf{N\_Genes} & \textbf{Sparsity} & \textbf{Batches} & \textbf{CellTypes} \\
\midrule
Immune & Recon/GRN & 32,484 & 12,303 & 88.15\% & 9 & 16 \\
Zheng68k & Recon/GRN & 68,579 & 32,738 & 98.34\% & \textit{N/A} & \textit{N/A} \\
\midrule
BMMC & Integration & 90,261 & 14,087 & 88.87\% & 12 & 45 \\
Perirhinal & Integration & 17,535 & 59,357 & 96.33\% & 2 & 10 \\
PBMC12k & Integration & 11,990 & 3,346 & 86.32\% & 2 & 9 \\
COVID-19 & Integration & 20,000 & 1,200 & 89.52\% & 2 & 39 \\
\midrule
MS & Fine-tuning & 21,312 & 3,000 & 89.28\% & 2 & 18 \\
hPancreas & Fine-tuning & 14,818 & 3,000 & 87.06\% & 2 & 14 \\
Myeloid & Fine-tuning & 13,178 & 3,000 & 80.84\% & 2 & 21 \\
Myeloid\_b & Fine-tuning & 9,926 & 3,000 & 81.16\% & 2 & 7 \\
\midrule
Cell Lines & Zero-shot & 9,531 & 32,738 & 89.80\% & 3 & 2 \\
DC & Zero-shot & 576 & 26,593 & 80.98\% & 2 & 4 \\
HumanPBMC & Zero-shot & 15,476 & 33,694 & 95.20\% & 2 & 9 \\
MCA & Zero-shot & 6,954 & 15,006 & 91.22\% & 2 & 11 \\
PBMC & Zero-shot & 18,868 & 6,998 & 95.32\% & 2 & 7 \\
PBMC\_368K & Zero-shot & 4,638 & 14,236 & 94.93\% & 2 & 8 \\
Pancrm & Zero-shot & 14,767 & 15,558 & 77.85\% & 5 & 15 \\
\midrule
Adamson & Perturbation & 68,603 & 5,060 & 79.32\% & \textit{N/A} & 1 \\
Norman & Perturbation & 91,205 & 5,045 & 91.89\% & \textit{N/A} & 1 \\
\bottomrule
\end{tabular}}
\end{table*}

\section{Model Architecture and Implementation Details}
\label{app:arch}

\subsection{Joint Embedding Formulation}
\label{app:joint_embedding}

Each serialized position corresponds to a gene identity $g\in\mathcal{V}$ and a continuous value $v\in\mathbb{R}$.
Let $d$ be the hidden dimension.
The input representation is
\begin{equation}
h_{\text{input}} = \mathrm{Emb}_{\text{id}}(g) + \mathrm{MLP}_{\text{val}}(v),
\label{eq:joint_embedding}
\end{equation}
where $\mathrm{Emb}_{\text{id}}:\mathcal{V}\rightarrow\mathbb{R}^{d}$ is a learnable embedding table,
and $\mathrm{MLP}_{\text{val}}:\mathbb{R}\rightarrow\mathbb{R}^{d}$ is a 2-layer perceptron projecting scalar values to $\mathbb{R}^{d}$, e.g.,
\begin{equation}
\mathrm{MLP}_{\text{val}}(v)=W_2 \sigma(W_1 v + b_1)+b_2,
\label{eq:mlp_val}
\end{equation}
with nonlinearity $\sigma(\cdot)$ (e.g., SiLU).

\subsection{Pre-Norm Transformer Block Dynamics}
\label{app:prenorm}

Let $x_l\in\mathbb{R}^{L\times d}$ denote the sequence representation at layer $l$.
Using Pre-Norm with RMSNorm, the $l$-th block computes
\begin{equation}
x'_{l} = x_l + \mathrm{Attention}\left(\mathrm{RMSNorm}(x_l)\right),
\label{eq:prenorm_attn}
\end{equation}
\begin{equation}
x_{l+1} = x'_l + \mathrm{FFN}\left(\mathrm{RMSNorm}(x'_l)\right).
\label{eq:prenorm_ffn}
\end{equation}
No causal mask is used, enabling bidirectional conditioning across the serialized gene sequence.

\subsection{Component Formulations}
\label{app:components}

\paragraph{Multi-head attention.}
For head dimension $d_h=d/H$ with $H$ heads, define $Q=XW_Q$, $K=XW_K$, $V=XW_V$.
RoPE is applied to $(Q,K)$ (below).
The attention output is
\begin{equation}
\begin{split}
\mathrm{Attention}(X) &= \mathrm{Concat}\left(\mathrm{head}_1,\ldots,\mathrm{head}_H\right)W_O, \\
\mathrm{head}_h &= \mathrm{softmax}\left(\frac{\tilde Q_h \tilde K_h^\top}{\sqrt{d_h}}\right)\tilde V_h.
\end{split}
\label{eq:mha}
\end{equation}

\paragraph{SwiGLU.}
The feed-forward network uses SwiGLU gating:
\begin{equation}
\mathrm{SwiGLU}(x)=\left(\mathrm{SiLU}(xW_g)\odot(xW_u)\right)W_d.
\label{eq:swiglu}
\end{equation}

\paragraph{RMSNorm.}
For $\epsilon>0$ and learnable scale $\gamma\in\mathbb{R}^{d}$,
\begin{equation}
\mathrm{RMSNorm}(x)=\frac{x}{\sqrt{\mathrm{Mean}(x^2)+\epsilon}}\odot\gamma.
\label{eq:rmsnorm}
\end{equation}

\paragraph{RoPE.}
Let $\Theta$ be the RoPE base (we use $\Theta=10000$).
For each head and each position $\mathrm{pos}$, RoPE applies a block-diagonal rotation to pairs of coordinates.
For $j=0,\dots,\frac{d_h}{2}-1$, define angular frequency $\omega_j=\Theta^{-2j/d_h}$ and
\begin{equation}
R_{\Theta,\mathrm{pos}}^{(j)} =
\begin{bmatrix}
\cos(\omega_j,\mathrm{pos}) & -\sin(\omega_j,\mathrm{pos})\\
\sin(\omega_j,\mathrm{pos}) & \cos(\omega_j,\mathrm{pos})
\end{bmatrix}.
\label{eq:rope_rotation}
\end{equation}
Applying $R_{\Theta,\mathrm{pos}}$ to each pair yields $\tilde Q=R_{\Theta,\mathrm{pos}}Q$ and $\tilde K=R_{\Theta,\mathrm{pos}}K$.

\paragraph{Dual output heads.}
Let $h_L^i\in\mathbb{R}^{d}$ be the final-layer hidden state at position $i$.
scDiVa predicts gene identity logits and a scalar value:
\begin{equation}
\begin{split}
\hat y_{\text{id}}^i &= \mathrm{Linear}_{\text{gene}}(h_L^i)\in\mathbb{R}^{|\mathcal{V}|}, \\
\hat y_{\text{val}}^i &= \mathrm{Linear}_{\text{val}}(h_L^i)\in\mathbb{R}.
\end{split}
\label{eq:dual_heads}
\end{equation}

\subsection{Hyperparameters}
\label{app:hparams}
Tables \ref{tab:model_config} and \ref{tab:train_config} list hyperparameters selected for efficiency. The architecture employs SwiGLU , RMSNorm, and RoPE to ensure stability and expressivity, optimized via AdamW.

\begin{table}[H]
\centering
\caption{Model configuration.}
\label{tab:model_config}
\begin{tabular}{l r}
\toprule
\textbf{Item} & \textbf{Value} \\
\midrule
\# Layers & 12 \\
Hidden dim $d$ & 512 \\
\# Attention heads $H$ & 8 \\
FFN hidden dim & 2048 \\
Vocabulary size $|\mathcal{V}|$ & 41,818 \\
Max sequence length $L_{\max}$ & 1200 \\
Normalization & RMSNorm ($\epsilon=10^{-5}$) \\
Activation & SwiGLU \\
RoPE base $\Theta$ & 10000 \\
\bottomrule
\end{tabular}
\end{table}

\begin{table}[t]
\centering
\caption{Training configuration.}
\label{tab:train_config}
\begin{tabular}{l r}
\toprule
\textbf{Item} & \textbf{Value} \\
\midrule
Global batch size & 768 \\
Optimizer & AdamW \\
Loss weight $\lambda$ (value term) & 10.0 \\
Time sampling & $t\sim \mathrm{Unif}(0,1)$ \\
\bottomrule
\end{tabular}
\end{table}

\section{Detailed Evaluation Metrics}
\label{app:metrics}

\subsection{Gene Reconstruction Metrics}
\label{app:recon_metrics}

Let $g_{1:L}$ be the ground-truth serialized gene identity sequence and $\hat g_{1:L}$ the reconstructed sequence.
Let $v_{1:L}$ and $\hat v_{1:L}$ be the corresponding expression values.

\paragraph{L-Dist (Wasserstein-1 over ranks).}
Let $\pi(g)$ be the rank (position index) of gene $g$ in the ground-truth list and $\hat\pi(g)$ the rank in the predicted list (restricted to the evaluated set).
Define
\begin{equation}
\mathrm{L\text{-}Dist}(g,\hat g)
=
\frac{1}{L}\sum_{k=1}^{L}\left|k - \hat\pi(g_k)\right|,
\label{eq:ldist}
\end{equation}
which is equivalent to the 1-Wasserstein (earth mover) distance between the two uniform measures over ranks with ground cost $|i-j|$ under the induced permutation.

\paragraph{BLEU for gene sequences.}
Treat $g_{1:L}$ and $\hat g_{1:L}$ as token sequences.
For $n$-grams ($n=1,\dots,N$), define clipped precision $p_n$ and the brevity penalty $\mathrm{BP}$.
BLEU-$N$ is
\begin{equation}
\mathrm{BLEU}\text{-}N = \mathrm{BP}\cdot \exp\left(\sum_{n=1}^{N} w_n \log p_n\right),
\label{eq:bleu}
\end{equation}
with weights $w_n$ (typically $w_n=\frac{1}{N}$).

\paragraph{Spearman correlation for values.}
Compute Spearman's rank correlation between $\hat v$ and $v$ across evaluated genes:
\begin{equation}
\rho_{\mathrm{sp}} = 1 - \frac{6\sum_{j=1}^{m} d_j^2}{m(m^2-1)},
\label{eq:spearman}
\end{equation}
where $m$ is the number of evaluated genes and $d_j$ is the difference between ranks of $\hat v_j$ and $v_j$.

\subsection{Integration Metrics}
\label{app:integration_metrics}

Let $z_i\in\mathbb{R}^{d_z}$ be the learned cell embedding for cell $i$, with batch label $b_i$ and biological label (cell type) $c_i$.
Let $\mathrm{ASW}(\cdot)$ denote the average silhouette width computed from pairwise distances in embedding space.

\paragraph{Avg-Batch.}
We quantify batch mixing via
\begin{equation}
\mathrm{Avg\text{-}Batch} = 1 - \mathrm{ASW}_{\text{batch}},
\label{eq:avg_batch}
\end{equation}
where $\mathrm{ASW}_{\text{batch}}=\mathrm{ASW}({z_i},{b_i})$ (lower silhouette for batch labels indicates better mixing).

\paragraph{Avg-Bio.}
We quantify biological conservation by combining clustering agreement and within-type compactness.
Let $\widehat c_i$ be cluster assignments obtained from ${z_i}$.
Define normalized mutual information $\mathrm{NMI}(\widehat c,c)$, adjusted Rand index $\mathrm{ARI}(\widehat c,c)$, and biological silhouette $\mathrm{ASW}_{\text{bio}}=\mathrm{ASW}({z_i},{c_i})$.
Then
\begin{equation}
\mathrm{Avg\text{-}Bio} = \frac{1}{3}\left(\mathrm{NMI}+\mathrm{ARI}+\mathrm{ASW}_{\text{bio}}\right).
\label{eq:avg_bio}
\end{equation}
\begin{figure}[t]
    \centering
    \includegraphics[width=\columnwidth]{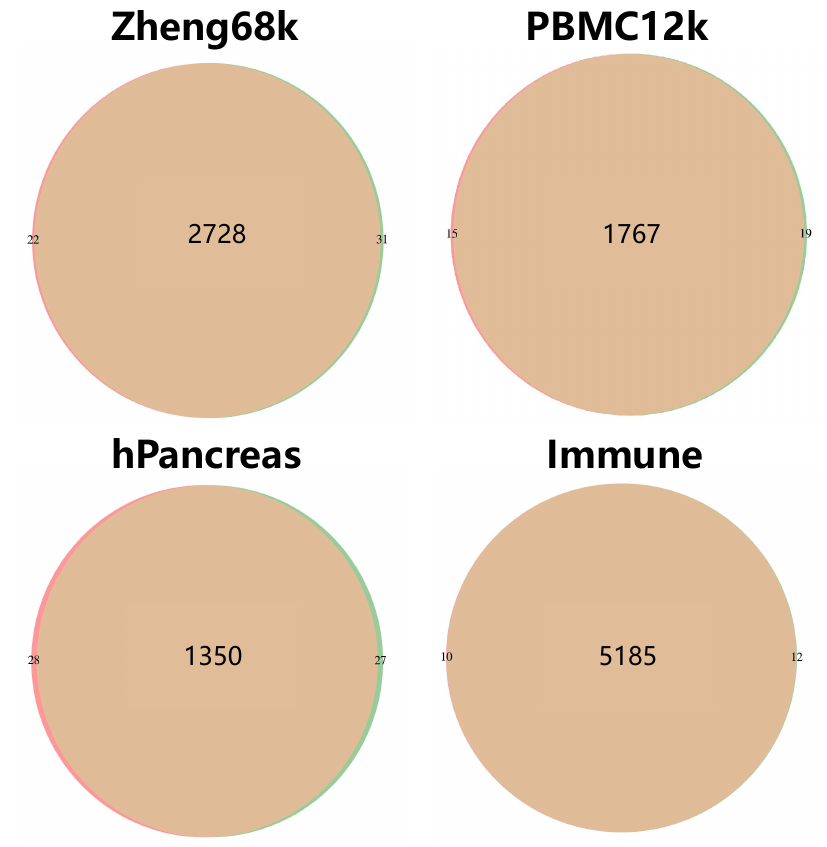}
    \caption{\textbf{Gene Space Reconstruction Overlap.} Venn diagrams illustrating the intersection of effectively reconstructed gene sets across Zheng68k, PBMC12k, hPancreas, and Immune datasets. The high intersection numbers demonstrate the model's ability to capture core transcriptomic features consistently across varying biological contexts.}
    \label{fig:app_recon}
\end{figure}
\begin{figure*}[t]
    \centering
    \includegraphics[width=\textwidth]{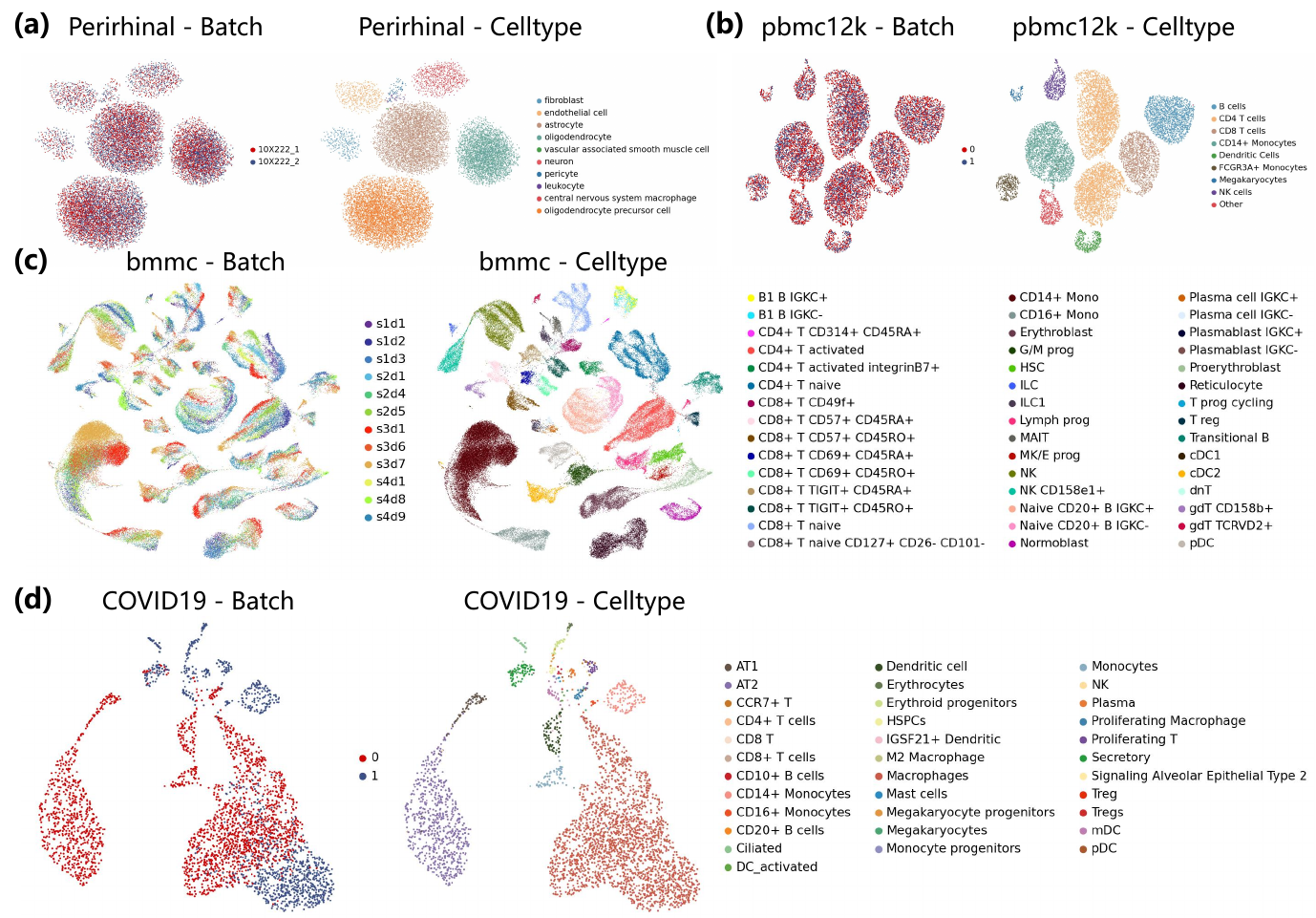}
    \caption{\textbf{Extended Multi-batch Integration Visualizations.} UMAP projections for (a) Perirhinal Cortex, (b) PBMC12k, (c) BMMC, and (d) COVID-19 datasets. Left columns display cells colored by batch ID, demonstrating effective mixing and removal of technical noise. Right columns display cells colored by cell type, confirming the preservation of biological identity and structural heterogeneity.}
    \label{fig:app_batch}
\end{figure*}

\subsection{Perturbation Metric}
\label{app:perturb_metrics}

Let $\Delta_j$ be the ground-truth perturbation-induced change for gene $j$ (e.g., mean treated minus mean control), and $\widehat\Delta_j$ the model prediction.
Define signal-sensitive weights proportional to gene expression (or magnitude of change), e.g.
\begin{equation}
w_j = \frac{\bar v_j + \alpha}{\sum_{k=1}^{G}(\bar v_k+\alpha)},
\label{eq:wmse_weights}
\end{equation}
where $\bar v_j$ is an average expression statistic (e.g., control mean) and $\alpha>0$ stabilizes low-expression genes.
The weighted MSE is
\begin{equation}
\mathrm{WMSE} = \sum_{j=1}^{G} w_j (\widehat\Delta_j - \Delta_j)^2.
\label{eq:wmse}
\end{equation}

\section{Additional Experimental Results}
\label{app:additional_results}

\textbf{Note on Inference Sampling:} Due to the high computational cost associated with multi-step diffusion sampling on large-scale benchmarks, all experimental results reported in this section are derived from single-step generation (i.e., direct prediction from the masked state), unless explicitly stated otherwise.

In this appendix, we provide extended qualitative and quantitative analyses to support the findings presented in the main text. We organize these results by the downstream tasks: Gene Reconstruction, Multi-batch Integration, Cell Type Annotation, Perturbation Prediction, and Gene Regulatory Network Inference.
\subsection{Extended Gene Reconstruction Analysis}
\label{app:gene_recon}

To further validate the generative capacity of scDiVa, we analyzed the intersection of reconstructed features across diverse datasets. Figure \ref{fig:app_recon} displays Venn diagrams illustrating the overlap of highly variable genes effectively reconstructed by the model across the Zheng68k, PBMC12k, hPancreas, and Immune datasets.

The substantial overlap indicates that scDiVa captures universal transcriptomic dependencies that are invariant across tissues. Specifically, the model maintains high fidelity not only for dataset-specific markers but also for shared housekeeping modules. This suggests that the Rank-Value reconstruction objective successfully prevents mode collapse and ensures that the generated expression profiles respect the underlying biological manifold of distinct datasets.
\begin{figure*}[t]
    \centering
    \includegraphics[width=\textwidth]{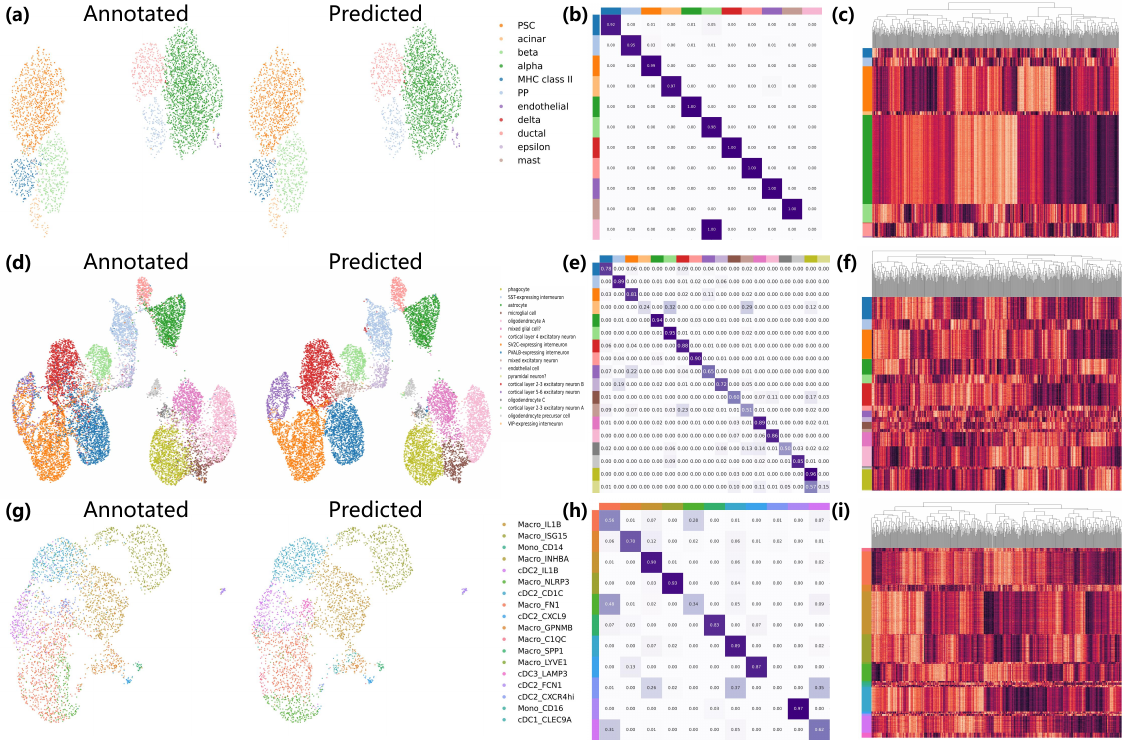}
    \caption{\textbf{Detailed Cell Type Annotation Analysis.} (a, d, g) UMAP comparisons of Ground Truth (Annotated) versus scDiVa predictions. (b, e, h) Confusion matrices showing classification accuracy, where the diagonal represents correct predictions. (c, f, i) Hierarchical clustering heatmaps of the learned features, organized by predicted cell types.}
    \label{fig:app_annot}
\end{figure*}
\subsection{Extended Multi-batch Integration}
\label{app:batch_viz}

While the main text focuses on the Immune dataset, we evaluated scDiVa's integration performance on four additional challenging scenarios: Perirhinal Cortex (brain tissue), PBMC12k (blood), BMMC (bone marrow), and COVID-19 (pathological lung tissue). 

As shown in Figure \ref{fig:app_batch}, scDiVa achieves a robust balance between batch mixing and biological conservation across all scenarios.

For the \textbf{Perirhinal Cortex} and \textbf{PBMC12k} datasets (Panels a, b), the UMAP projections colored by 'Batch' show thorough intermingling of samples (red and blue dots), indicating effective removal of technical artifacts. Conversely, the 'Cell type' plots reveal distinct, compact clusters, preserving the separation between neuronal subtypes and immune populations.

In the \textbf{BMMC} dataset (Panel c), which contains complex developmental trajectories, scDiVa preserves the continuum of differentiation (e.g., from HSCs to erythroid progenitors) while merging samples from different donors. This is evidenced by the unified trajectory in the batch plot and gradient separation in the cell type plot.

Finally, for \textbf{COVID-19} (Panel d), the dataset introduces severe pathological heterogeneity. Despite strong disease-induced shifts, scDiVa successfully aligns the control and disease batches without erasing the disease-specific cell states, such as activated macrophages and T cells.

\subsection{Detailed Cell Type Annotation}
\label{app:cell_annot}

To provide granular insight into the classification performance, we visualize the confusion matrices, predicted embeddings, and hierarchical heatmaps for representative fine-tuning tasks.

Figure \ref{fig:app_annot} presents these results across three distinct benchmarks. The comparison between ``Annotated" (Ground Truth) and ``Predicted" UMAPs (Panels a, d, g) confirms that the model-learned latent space aligns closely with human-curated labels. 

The \textbf{Confusion Matrices} (Panels b, e, h) exhibit strong diagonal dominance, indicating high classification accuracy. Notably, off-diagonal errors are primarily concentrated among biologically related subtypes (e.g., different excitatory neuron layers in Panel e), reflecting the subtle transcriptomic differences rather than model failure.

Furthermore, the \textbf{Hierarchical Heatmaps} (Panels c, f, i) illustrate the distinct expression signatures learned for each predicted class. The block-diagonal structure confirms that scDiVa extracts discriminative features that are consistent within cell types and distinct between them.

\subsection{Perturbation Prediction Hit Rate}
\label{app:perturb}

Beyond the correlation and MSE metrics reported in the main text, predicting the exact set of differentially expressed (DE) genes is critical for experimental prioritization. Figure \ref{fig:app_hitrate} illustrates the Hit Rate at top-$k$ for the Norman dataset. 

The Hit Rate metric quantifies the proportion of the true top-$k$ most responsive genes that are correctly identified by the model. scDiVa consistently outperforms baseline methods across varying $k$ values. This suggests that even when the exact magnitude of expression change is challenging to predict, our model correctly ranks the genes most affected by the perturbation, providing valuable candidates for downstream wet-lab validation.

\begin{figure}[h]
    \centering
    \includegraphics[width=\columnwidth]{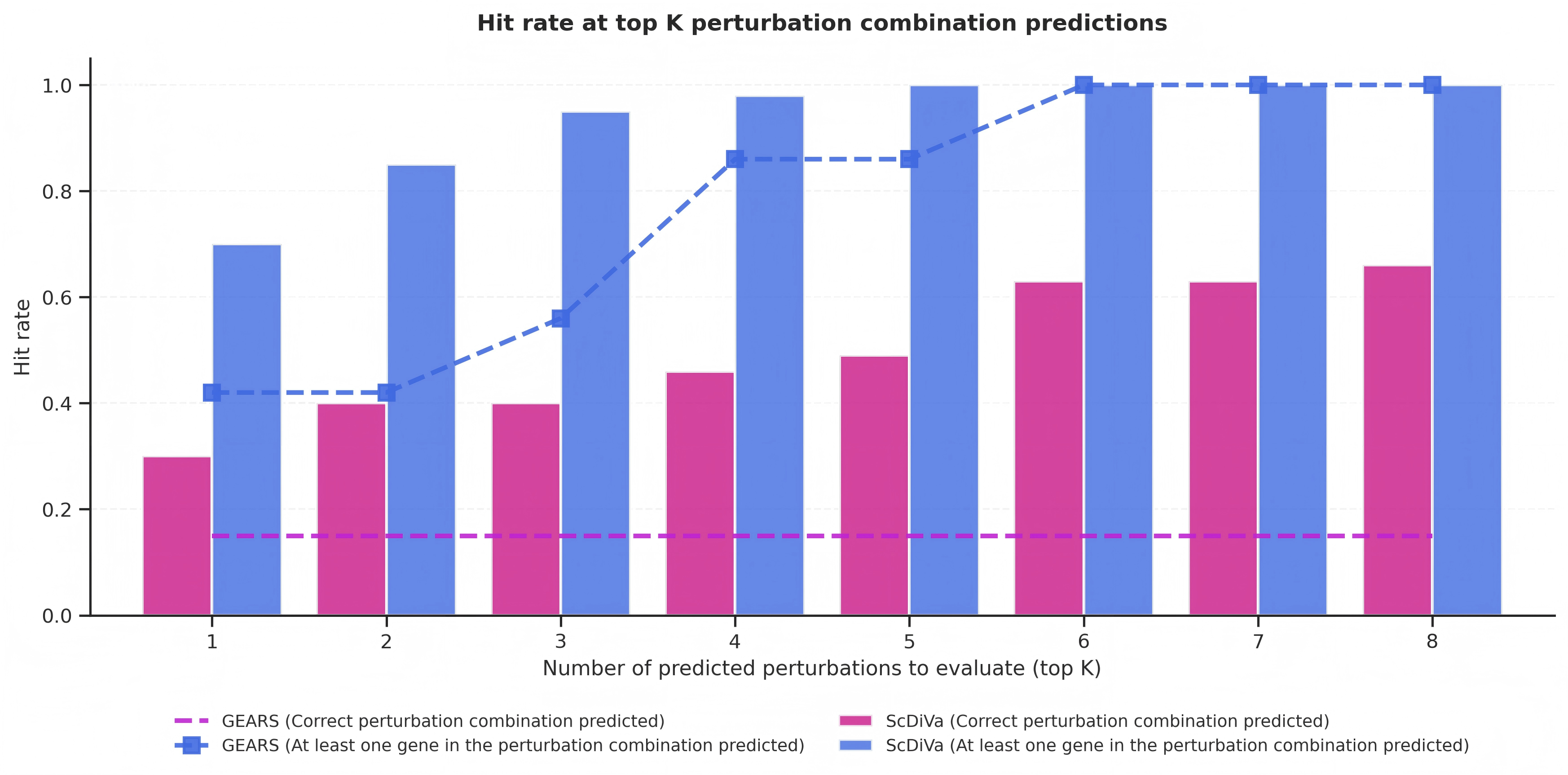}
    \caption{\textbf{Hit Rate at Top-$k$ for Perturbation Prediction.} Evaluation on the Norman dataset showing the proportion of true differentially expressed genes successfully retrieved by the model within the top $k$ predictions. scDiVa demonstrates superior ranking capability compared to baselines.}
    \label{fig:app_hitrate}
\end{figure}

\subsection{Inferred Gene Regulatory Networks}
\label{app:grn_viz}

The attention mechanisms within scDiVa allow for the extraction of global gene regulatory logic. In the main text, we highlighted the SPI1 regulon. Here, we extend this analysis to other critical biological pathways, verifying the model's ability to capture diverse regulatory motifs.

Figure \ref{fig:app_grn} displays the inferred gene regulatory networks for four distinct biological modules:
\vspace{-1mm}
\begin{itemize}[left=-1.5pt]
    \item[] \textbf{(a) TNF Superfamily:} The model recovers ligand-receptor pairs, specifically the interaction between TNFSF13B and TNFRSF13B.
    \item[] \textbf{(b) Interleukin Family (IL-1):} The graph exhibits high connectivity within the IL-1 family, consistent with the coordinated expression patterns of pro-inflammatory cytokines.
    \item[] \textbf{(c) CD Markers:} We observe distinct co-expression clusters corresponding to cell-surface markers.
    \item[] \textbf{(d) HLA Complex:} The attention mechanism identifies the dense correlation structure among MHC Class II genes (HLA-DR, -DQ, -DP).
\end{itemize}
Edge weights represent raw attention scores, demonstrating that the model captures known co-regulation patterns and protein-protein associations directly from the transcriptomic input.

\begin{figure}[t]
    \centering
    \includegraphics[width=\columnwidth]{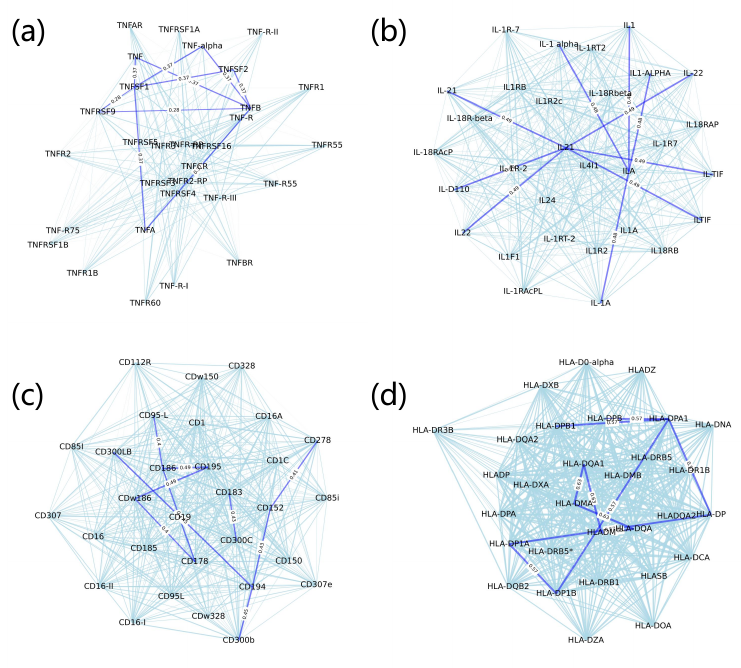}
    \caption{\textbf{Extended Gene Regulatory Network Inference.} Network graphs constructed from scDiVa attention weights for specific gene families: (a) TNF superfamily, (b) Interleukins (IL-1 family), (c) CD surface markers, and (d) HLA complex. Thicker edges indicate stronger regulatory attention weights, revealing biologically validated co-expression modules.}
    \label{fig:app_grn}
\end{figure}

\end{document}